\documentclass[12pt]{article}
\usepackage{rotating,amsmath, amssymb, amsthm}
\usepackage{amsfonts, verbatim,url}
\usepackage{epsfig,url, psfrag,appendix}
\usepackage{graphicx, multirow,graphics}
\usepackage{setspace}
\usepackage[left=2.8cm,top=2.5cm,right=2.5cm,bottom=2.8cm]{geometry}
\usepackage{natbib}
\usepackage{float}
\floatstyle{ruled}

\usepackage{epsfig,psfrag,url,multirow,verbatim}
\usepackage{graphics,enumerate,subfig}
\usepackage{graphicx}
\usepackage{amsmath, amssymb}
\usepackage{verbatim}
\usepackage{listings}
\usepackage{amsfonts}
\usepackage{algorithmic}
\usepackage{algorithm}
\usepackage{bm}

\newtheorem{theorem}{Theorem}
\newtheorem{lem}{Lemma}
\newtheorem{lemma}{Lemma}

\newtheorem{remark}{Remark}
\newcommand{\B}{\boldsymbol}
\newcommand{\M}{\mathbf}
\newcommand{\MT}{\textbf}

\newcommand{\sgn}{\operatorname{sgn}}
\newcommand{\s}{\mathbf S}

\newcommand{\reminder}[1]{{\bf ** #1 **}}
\newcommand{\commentout}[1]{}

\def\sglasso{{\sc glasso} }
\def\sSMACS{ {\sc smacs} }
\def\GL{\sglasso}

\DeclareMathOperator*{\argmin}{arg\,min}

\DeclareMathOperator*{\mini}{minimize}

\newcommand{\pine}{{\textsc \small{PINE-GL}} }
\newcommand{\pex}{{\textsc \small{PEX-GL}} }
\newcommand{\pgr}{{\textsc \small{PGR-GL}} }

\newcommand{\OMOD}{{\textsc \small{PINE-GL} }}

\begin{document}

\title{A Flexible, Scalable and Efficient Algorithmic Framework for \emph{Primal} Graphical Lasso}


\author{Rahul Mazumder\\
Department of Statistics, Stanford University, Stanford, CA\\
email: { \tt rahulm@stanford.edu} 
\and
Deepak K. Agarwal \\
Yahoo! Research, 4401 Great America Parkway, Santa Clara.\\
email: {\tt dagarwal@yahoo-inc.com}}
\date{Submitted for publication on 10-11-2011}
\maketitle

\begin{abstract}
We propose a scalable, efficient and statistically motivated computational framework for Graphical Lasso \citep{FHT2007a} --- 
a covariance regularization framework that has received significant attention in the statistics community over the past few years. 
Existing algorithms have trouble in scaling to dimensions larger than a thousand. 
Our proposal significantly enhances the state-of-the-art for such moderate sized problems  
and gracefully scales to larger problems where other algorithms become practically infeasible. 
This requires a few key new ideas. We operate on the primal problem and use a subtle variation of block-coordinate-methods which drastically reduces the computational 
complexity by orders of magnitude. We provide rigorous theoretical guarantees on the convergence and complexity of our algorithm and 
demonstrate the effectiveness of our proposal via experiments. 
We believe that our framework extends the applicability of Graphical Lasso to large-scale modern applications like bioinformatics, collaborative filtering and social networks, among others. 
\end{abstract}

\noindent \textbf{keywords} \\
Graphical Lasso, $\ell_1$ regularization / LASSO, sparse inverse covariance selection, large scale convex optimization, convergence analysis, covariance estimation, positive definite matrices

\section{Introduction} 
\label{sec:intro} 
\paragraph{Problem Description}
Let $\s_{p\times p}$ denote a $p$-dimensional sample covariance matrix obtained through i.i.d samples from  
a multivariate Gaussian distribution with (unknown) covariance $\B{\Sigma}$  
and precision matrix
$\B{\Sigma}^{-1}$.
 The negative log-likelihood is given by:   
\begin{equation}  \label{eqn-likhd}  
f(\B{\Theta}) := -  \log\det \B{\Theta} + \langle \s, \B{\Theta}\rangle \;\;\text{on $\B{\Theta} \succ \mathbf{0}$}, 
\end{equation}  
where  $\langle \s, \B{\Theta}\rangle:=\mathrm{tr}(\s\B{\Theta})$ and $\B\Theta$ corresponds to the precision matrix. 
The MLE (when it exists)  
is $\widehat{\B{\Theta}}=\s^{-1}$, but this estimator has high variance unless the sample size $n$ is large relative  
to the dimension $p$. This makes the 
MLE a not-so-useful estimator of the covariance/precision matrix. In such high dimensional problems regularization (smoothing) is imperative to
obtain reliable estimates.  In fact, for the Gaussian distribution 
the precision matrix \citep{cox-W-96,Laur1996} captures conditional dependencies among variables where absence of an edge (zero entry in the precision matrix) 
implies conditional independence. Hence, taking recourse to smoothing methods that induce sparsity is attractive.
In addition to producing shrinkage estimators, a sparse precision graph leads to interpretable models and also provides  
model compression. In the context of learning large-scale graphs it is often undesirable from the point of view of 
computational/storage considerations to learn a model with all possible $p^2$ edges present. 
Surprisingly enough, for large scale problems i.e. with $p \approx 10^4 - 10^6$, sparse precision graphs are 
computationally tractable, whereas their dense counterparts are not \citep[see for details,]{MH-GL-11}. 

The $\ell_{1}$ regularization \citep{FHT2007a,BGA2008,yuan_lin_07,MB2006} is often used in this context
since it performs smoothing, induces sparsity, adds interpretation and forms an effective 
model selection procedure. This is popularly known as \emph{sparse inverse covariance selection}  or the Graphical Lasso  
and is obtained as a solution to the following regularized criterion:
\begin{equation}\label{crit1}  
\mini_{\B{\Theta} \succ \mathbf{0}} \;\;g(\B{\Theta}) := f(\B{\Theta}) + \lambda \sum_{ij}|\theta_{ij}| 
\end{equation} 
where $\lambda > 0$ is the amount of regularization imposed on the entries of the precision matrix $\B\Theta=((\theta_{ij}))$.
Equation~(\ref{crit1}) is a convex optimization problem (Semi-Definite Program aka SDP) in the variable $\B\Theta$.
The class of models described through equation~(\ref{crit1}) has already gained
widespread interest in many statistical applications like
biostatistics, functional magnetic resonance imaging, network
analysis, collaborative filtering \citep{FHT2007a,neuro-alzh-10,AZM-11}, and many more. 
Considerable progress has also been made in studying the statistical properties of
the estimator and its variants \citep{ravi-11,fan-09}.  
We also note that the optimization in (\ref{crit1}) is often used in a more non-parametric fashion \citep{AZM-11,MH-GL-11} for
any positive semidefinite input matrix $\s$, not necessarily a sample
covariance matrix from a MVN sample. 

\paragraph{Context and Background}
Interior point
methods for solving (\ref{crit1})
scale poorly with increasing dimensions and quickly become infeasible
for problem sizes around a hundred. For scalability purposes,
first order methods relying on gradient information instead of Hessian (i.e. second order methods)
become almost
imperative \citep{nest_03}. Over the past few years there has been substantial 
interest in developing such specialized scalable solvers for
(\ref{crit1}) 
\citep{FHT2007a,BGA2008,Lu:09,Lu:10,Scheinberg_Ma_Goldfarb_2010,Yuan_2009,boyd-admm}. 
However, existing solvers have difficulty in
scaling to problems with $p > 1000$ --- precluding the
wide-spread use of these methods in modern day applications like
collaborative filtering, graph mining, web-applications, large microarray data and 
other high dimensional problems.

There have been other interesting formulations to sparse precision
matrix estimation
\citep[for example,]{Rothman01092010,sp_l1_cai-11,MB2006,FHT-GL-10}. The formulation of \citet{Rothman01092010} is non-convex.
\citet{sp_l1_cai-11} consider a linear programming approach where the
precision matrix estimate \emph{need not} be positive definite.
Pseudo-likelihood based approaches \citep{FHT-GL-10} do not ensure
positive definiteness of the matrices.  In this paper we focus on
(\ref{crit1}).

\paragraph{Motivation}
Several large scale covariance selection problems require algorithms that produce reasonably 
accurate approximations to the optimal solution of (\ref{crit1}) within a certain time limit or equivalently, a limit on the computational budget \citep{boto-08}. 

In fact, for large scale problems, under computational constraints an approximate solution to (\ref{crit1}) is often the only feasible option. But for several applications,
other than speed and scalability, it is necessary for the approximate
solution to retain crucial properties of the optimal solution like sparsity and positive definiteness.
One such application of large scale covariance selection was recently explored in the work of \citet{AZM-11}. The authors  
used a sparse inverse covariance regularization to estimate the covariance matrix of high dimensional random
effects in a multi-level hierarchical model. The paper explored prediction problems in 
recommender systems where the goal was to predict responses on unobserved user-item cells in a large matrix using
responses on observed user-item pairs. Each user $u$ is assigned an $M$ dimensional random vector $\bm{\phi}_u$ that represents
the user's latent affinity to $M$ items. $\bm{\phi}_{u}$'s are assumed to be drawn from a multivariate Guassian prior with unknown covariance. The covariance
is estimated via a E-M framework using an $\ell_1$ regularization on the elements of precision matrix.

The use of a sparse inverse covariance regularization in the paper 
\citep{AZM-11} led to a model with better predictive accuracy compared
to other state-of-the-art methods. Since the estimation is based on
an E-M strategy, it requires strictly positive
definite estimates of the covariance and precision matrix. Indeed,
an optimization of the form (\ref{crit1}) 
needs to be conducted in the M-step --- hence it is enough to terminate the process early
without complete optimization. Early stopping along with sparsity in the precision 
matrix leads to a drastic reduction in computation time. 
The key property of
covariance estimation required for the method to work is the ability
to return both the precision matrix and its exact inverse i.e. the covariance matrix.
These properties, apparently  are not possessed by existing algorithms for (\ref{crit1}). 
This may have precluded the use of sparse inverse covariance for estimating covariances in high dimensional random-effects
model. 
We note that the strategy used in \citet{AZM-11} using the model fitting
method described in this paper is general and
can be used to model covariance in other high-dimensional multivariate
random effects model that arise in applications like spatial
statistics~\citep{besag}, social networks~\citep{hoff,hoffrafteryhandcock}, and many more.  

In the scenarios described above, we want a `flexible' fitting algorithm for
(\ref{crit1}) such that:
\vskip -3pt
\begin{itemize}
\item It can deliver a solution of arbitrary accuracy to (\ref{crit1}) --- the accuracy depending upon demands of the user/ application.    
\item Even if a low accuracy solution is desired, the algorithm upon exiting should return a 
sparse and positive definite $\B\Theta$ and its inverse $\B\Theta^{-1}$ --- fundamental ingredients for relevant 
statistical model fitting procedures. 
\item The computational complexity per iteration of the algorithm is cheap enough to solve large scale problems.
\item It readily adapts to warm-starts for computing a path of solutions on a grid of $\lambda$ values.
\end{itemize}
We believe estimation procedures for inverse covariance described in this paper  
will make it routine to apply large scale
covariance selection to high-dimensional multivariate data. 

\paragraph{Our Approach}
We provide a brief outline of our approach and the salient features that make it different from other existing algorithms.
Many of the sophisticated state-of-the-art algorithms 
\citep{BGA2008,Lu:09,Lu:10,Scheinberg_Ma_Goldfarb_2010,boyd-admm} designed to solve (\ref{crit1}) 
perform expensive operations like matrix inversions / eigen-decompositions 
on the \emph{entire} matrix at every iteration requiring $O(p^3)$, which is clearly prohibitive for large problems. 
We take a different route by pursuing row/column block-coordinate based methods that cyclically update the estimates of
one row/column at a time fixing the others at their latest values.  
Although \citet{FHT2007a,BGA2008} also pursue block-coordinate methods, our approach differs in a few very important ways.

First, while we operate on the primal (where the primal variable is the precision matrix $\B\Theta$), 
\citet{FHT2007a,BGA2008} operate on the dual of (\ref{crit1}). 
The primal and dual problems have some subtle and important 
differences that need consideration for large scale statistical applications. 
Algorithms operating on the dual \citep[for example]{FHT2007a,Lu:10,BGA2008} do not return a sparse and positive definite precision matrix unless optimization is done to a very high 
degree of accuracy --- this may be prohibitive for large scale problems.
A more detailed discussion of this issue is provided in Section \ref{sec:related}.

Second, we track both the precision and the covariance matrix over iterations, 
and our row/column block-coordinate wise procedure 
\textit{does not perform a complete minimization} over the partial problems. 
This is a crucial observation since it reduces the row/column update cost from $O(p^3)$ to $O(p^2)$. 
Although the idea looks simple at first blush, such incomplete minimization over partial problems is not necessarily guaranteed to ensure a proper optimization
algorithm with convergence certificates. In Section \ref{sec:conv-analysis} we show that such a relaxation still guarantees convergence of our algorithm. 
To the best of our knowledge, such a convergence analysis is novel both in the statistics and optimization literature. 
\paragraph{Contributions}
We provide a summary of our main contributions before a detailed description of 
our approach. We propose a novel model fitting algorithm for a popular covariance selection method (\ref{crit1})  
that outperforms previous state of the art fitting algorithms for large problems with dimension $p \approx 1-3$
thousands. The Algorithm design requires new and novel ideas. Our proposal is particularly suited to compute a path of
solutions to (\ref{crit1}) by using warm-starts on a grid of $\lambda$
values. The performance gains are quite impressive when compared to
other existing algorithms for the same task as illustrated in Section~\ref{sec:expts}. In addition, our algorithm is
amenable to early stopping, provides a sparse and positive-definite solution, and scales
to very large scale problems that are impractical to fit using existing methods. 
We provide a novel proof of asymptotic (algorithmic) convergence analysis, analyze complexity of the method
and show the superiority of our methods through
large scale simulation and data analysis. Finally, we outline how our approach can accommodate other row/column separable convex 
regularizers.
\section{Algorithmic Framework} \label{our-method}
We now provide a detailed development of our fitting algorithm in this section, including convergence proof and computational
complexity analysis. We begin with notations.
\paragraph{Notations}
We denote the set of all $k \times k$ positive definite (respectively, positive semi-definite) matrices by $S^{++}_{k}$ 
(respectively $S^{+}_k$). We will write $A_{k \times k} \succ 0$ if $A \in S^{++}_{k}$, similarly $A \succeq 0$ implies
$A \in  S^{+}_{k}$.
For a matrix $A_{p \times p}$ we will denote its entries by $a_{ij}, i =1, \ldots, p; j=1, \ldots, p$.

For a vector $\M{u}$, the notation $\|\M{u}\|_2$ denotes the usual $\ell_2$ norm, $\|\M{u}\|_1$ denotes the $\ell_1$ norm.
For a matrix $\M{U}$, we will use 
$\|\M{U}\|_2$ to denote its spectral  norm i.e. the largest singular value of $\M{U}$.

\paragraph{Description of the Algorithm}
The block coordinate method operates by fixing a row/column index $i \in \{1,2,\ldots,p\}$, which without loss of generality, is assumed to be $p$. 
Partition the precision matrix $\B{\Theta}$ and the sample covariance matrix  $\s$ as follows:
\begin{eqnarray}\label{break-x} 
\B{\Theta} = \left( 
  \begin{array}{cc} 
    \B{\Theta}_{11} & \B{\theta}_{1p} \\ 
    \B{\theta}_{p1}  & \theta_{pp} \\ 
  \end{array} 
\right),    & \M{S} = \left(  
  \begin{array}{cc} 
    \mathbf{S}_{11} & \M{s}_{1p} \\ 
    \M{s}_{p1}  & s_{pp} \\ 
  \end{array} 
\right).    
\end{eqnarray} 
Using standard formulae for determinants of block-partitioned matrices we have:  
\begin{equation}\label{eq-block} 
\log\det(\B{\Theta})= \log\det(\B{\Theta}_{11}) +  \log(\theta_{pp} - \B{\theta}_{1p}'(\B{\Theta}_{11})^{-1}\B{\theta}_{1p}). 
\end{equation} 
Using the above, the part of $g(\B{\Theta})$ in equation~(\ref{crit1}) that depends upon the $p^{\mathrm{th}}$ row/column of $\B{\Theta}$ is given by: 
\begin{equation}\label{split}
g_p(\theta_{pp},\B{\theta}_{1p}) = - \log(\theta_{pp} - \B{\theta}_{1p}'(\B{\Theta}_{11})^{-1}\B{\theta}_{1p}) 
 +  2 \M{s}_{1p}'\B{\theta}_{1p} + (s_{pp} + \lambda)\theta_{pp} + 2\lambda\|\B{\theta}_{1p}\|_1 .
\end{equation}
Note that the positive-definiteness of $\B\Theta$ assures $\theta_{pp} \geq 0$.  
In (\ref{split}), the optimization variables are $\theta_{pp}$ and $\B{\theta}_{1p}$.  
Conventional forms of block coordinate descent \citep{Tseng01,FHT2007} when applied to this problem  
will require completely minimizing the function (\ref{split}) over the variables $\B{\theta}_{1p}$ and $\theta_{pp}$. 
Clearly for large problems an accurate optimization of this problem is quite computationally 
intensive, especially since this needs to be done several times across all rows/columns.  
We choose to deviate from this approach and propose to perform an \emph{inexact minimization} in the afore-mentioned stage.  
The fact that such a deviation still ensures a proper optimization procedure will be discussed later, for now we continue with the description of
the algorithm.

Minimizing the criterion (\ref{split}) with respect to $\theta_{pp}$ with other coordinates fixed gives:
\begin{equation}  
\widehat\theta_{pp} := \argmin_{\theta_{pp}} \;\;  g(\theta_{pp}, \B{\theta}_{1p}) \;\;=\;\;\;  1/(s_{pp} + \lambda) + 
 \B{\theta}_{1p}'(\B{\Theta}_{11})^{-1}\B{\theta}_{1p}.  \label{maxi-1} 
\end{equation}
The partially minimized objective (\ref{split}), w.r.t. $\theta_{pp}$ is given by:
\begin{equation}
\min_{\theta_{pp}}\; g_p(\theta_{pp},\B{\theta}_{1p}) = \log\det(s_{pp}+\lambda) +  2 \M{s}_{1p}'\B{\theta}_{1p} + 1\nonumber \\
 +  (s_{pp}+ \lambda)\B{\theta}_{1p}'(\B{\Theta}_{11})^{-1}\B{\theta}_{1p} + 
2\lambda\|\B{\theta}_{1p}\|_1. \label{margin-diag-1}
\end{equation}
Ignoring the constants independent of $\B{\theta}_{1p}$ above, we obtain an $\ell_1$ regularized quadratic\footnote{note that the problem is convex only if  $\B{\Theta}_{11}^{-1}\succeq 0$, which is the case by virtue of the positive definiteness of the precision matrices, as shown in 
Section~\ref{sec:props}} function, which we denote by: 
\begin{equation}
g_p(\B{\theta}_{1p}) =  \B{\theta}_{1p}'\{(s_{pp}+ \lambda)\B{\Theta}_{11}^{-1}\}\B{\theta}_{1p} + 2 \M{s}_{1p}'\B{\theta}_{1p} + 
2\lambda\|\B{\theta}_{1p}\|_1. \label{margin-diag-2}
\end{equation}
We propose to use \emph{one} sweep of cyclical coordinate-descent on this function $g(\B{\theta}_{1p})$, w.r.t. the variable $\B\theta_{1p}$. 

We now summarize the update rule described above. 
Fix an arbitrary $\widetilde {\B{\Theta}} \succ 0$,
\begin{equation}
\widetilde{\B{\Theta}} = \left( 
  \begin{array}{cc} 
    \widetilde{\B{\Theta}}_{11} & \widetilde{\B{\theta}}_{1p} \\ 
    \widetilde{\B{\theta}}_{p1}  & \widetilde{\theta}_{pp} \\ 
  \end{array} 
\right)
\end{equation}
and consider an increment in  
$\B{\Theta}$ around $\widetilde {\B\Theta}$ in the direction of the $p^{\mathrm{th}}$ row/column. This updates $\widetilde{\B\Theta}$ to 
$\widehat{\B\Theta}$:
\begin{equation}\label{def-omega} 
 \widehat{\B{\Theta}} \longleftarrow  \widetilde{\B{\Theta}} + (\B{\omega} \M{e}'_p + \M{e}_p\B{\omega}') 
\end{equation} 
 where 
$\M{e}_p$ is a vector in $\Re^{p}$, with all entries equal to 0 but the $p^{\mathrm{th}}$ entry equals to 1, 
$\B{\omega}=(\omega_1,\ldots,\omega_p)$ denotes the ``increment" in the direction of the $p^{\mathrm{th}}$ row/column. 
\begin{algorithm}\caption{Inner Block Inexact Coordinate-Descent}\label{algo:inner-cd}
\begin{enumerate}
\item\label{item-11} Initial value of $\B\Theta$ is $\widetilde{\B{\Theta}}$. Assign $\widehat{\B{\Theta}} =\widetilde{\B{\Theta}}$.
\item\label{item-12}  Update the entries $\omega_{1}, \ldots, \omega_{p-1}$ and also $\widetilde{\B{\theta}}_{1p}$, 
as in (\ref{inner-cd-update1}) and (\ref{inner-cd-update2}).
\item\label{item-13}  Update $\widehat{\omega}_{p}$ using the update-rule (\ref{update-diag1}).
Consequently change the $(p,p)^{\mathrm{th}}$ entry of $\widehat{\B{\Theta}}$ to $\widehat{\omega}_{p} + \widetilde{\theta_{pp}}$.
\end{enumerate}
\end{algorithm}
Using notation (\ref{def-omega}) and $g_p(\cdot), \cdot \in \Re^{p-1}$ 
as in (\ref{margin-diag-2}), the update rule in $\B{\omega}$ is given by:
\begin{eqnarray}
\widehat{\omega}_{i} = \argmin_{\omega_{i}} & g_p(\ldots, (\widehat{\B{\theta}}_{1p})_{i-1},(\widetilde{\B{\theta}}_{1p})_{i} + \omega_{i},(\widetilde{\B{\theta}}_{1p})_{i+1},\ldots) & \label{inner-cd-update1} \\
(\widehat{\B{\theta}}_{1p})_i \leftarrow (\widetilde{\B{\theta}}_{1p})_i + \widehat{\omega}_{i}, &\;\;
(\widehat{\B{\theta}}_{p1})_i \leftarrow (\widetilde{\B{\theta}}_{p1})_i + \widehat{\omega}_{i}, \;\; i= 1, \ldots, p-1. & 
\label{inner-cd-update2}
\end{eqnarray}
Observe that the update (\ref{inner-cd-update1}) is simply a soft-thresholding operation:
\begin{flalign} 
\widehat{\omega}_{i} = &\sgn(a_i)(|a_i| - \lambda)_+/b_i  - (\widetilde{\B{\theta}}_{1p})_{i},\;\;\; \text{where,}&  \label{soft-thre-1}\\
   a_i = & (\M{s}_{1p})_i  + (s_{pp}+ \lambda) \left ( \sum_{j < i}
(\B{\Theta}_{11}^{-1})_{ij}(\widehat{\B{\theta}}_{1p})_{j}  + \sum_{j > i} (\B{\Theta}_{11}^{-1})_{ij}(\widetilde{\B{\theta}}_{1p})_{j} \right),& b_i = (s_{pp}+ \lambda) (\B{\Theta}_{11}^{-1})_{ii} \nonumber 
\end{flalign}
Finally, upon updating the off-diagonal entries in $\widetilde{\B{\Theta}}$, the diagonal entry 
is updated using (\ref{maxi-1}): 
\begin{equation}\label{update-diag1}
\widehat{\omega}_{p} \leftarrow 1/(s_{pp} + \lambda) +  \widetilde{\B{\theta}}_{1p}'(\widetilde{\B{\Theta}}_{11})^{-1}\widetilde{\B{\theta}}_{1p} - 
\widetilde{\theta}_{pp}
\end{equation}
Overall, the above steps lead to the update formula:    
$\widehat{\B{\Theta}} \longleftarrow  \widetilde{\B{\Theta}} + (\widehat{\B{\omega}} \M{e}'_p + \M{e}_p\widehat{\B{\omega}}')$.

Note that the above operations require evaluations of the residual $a_i$. 
This requires computing at the onset the
full gradient vector 
of the smooth part in 
(\ref{margin-diag-2}) at $\widetilde{\B\theta}_{1p}$. 
When a coordinate of the vector $\widetilde{\B\theta}_{1p}$ gets updated, the entire gradient vector changes --- this update can be achieved in $O(p)$ flops. 
Note that in case $\widehat{\omega}_{i}=0$, no update 
is required. Hence, if on \emph{average} the number of non-zeros in $\widehat{\B{\theta}}_{p1}$, $\widetilde{\B{\theta}}_{p1}$ is 
$k$, then the update (\ref{inner-cd-update1})--(\ref{inner-cd-update2}) requires an overall 
$O(pk)$ flops which, for $k \ll p$ leads to a significant reduction over the cost of a dense matrix/ vector multiplication i.e. $O(p^2)$.
Algorithm~\ref{algo:inner-cd} summarizes the updating steps described above.

The above description was for updating the $p^{\mathrm{th}}$ row/column of the matrix $\B\Theta$. This needs to be done for every row/column --- one full sweep across the $p$ rows/columns defines one iteration of our algorithm.
We now describe the full version of our algorithm in Algorithm~\ref{algo:block-gl}, we name it: 
\MT{P}rimal \MT{IN}exact Minimization for \MT{G}raphical \MT{L}asso (\pine).
\begin{algorithm}[htpb]
\caption{Primal Inexact Minimization for Graphical Lasso (\pine)}\label{algo:block-gl}
Inputs: $\s,\lambda$. Initialization: $(\widetilde{\B\Theta},\widetilde{\B{\Theta}}^{-1})$.
\begin{enumerate} 
\item[1] For every row/column  $i \in \{1,2,\cdots,p,1,2,\cdots\}$, perform steps 2-3 till convergence.  
\item[2] Permute the matrix such that the $i^{\mathrm{th}}$ row/column is the $p^{\mathrm{th}}$ i.e. of the form (\ref{break-x}). \\
Obtain the matrix $(\widetilde{\B\Theta}_{11})^{-1}$ via rank-one updates (see Section \ref{track-inv}).\\
Update the matrix using Steps \ref{item-11} - \ref{item-13} in  Algorithm~\ref{algo:inner-cd} :  
$\widehat{\B{\Theta}} \leftarrow  \widetilde{\B{\Theta}}  + (\widehat{\B{\omega}} e'_p + e_p\widehat{\B{\omega}}')$

Obtain  $(\widehat{\B{\Theta}})^{-1}$ via rank-one-updates (see Section~\ref{track-inv}). 

Re-permute the matrix to get back the original rows/column indexing.
\item[3] Assign $(\widetilde{\B{\Theta}},(\widetilde{\B{\Theta}})^{-1})\leftarrow (\widehat{\B{\Theta}},(\widehat{\B{\Theta}})^{-1})$ 
\item[4] Upon convergence, the estimates at $\lambda$: 
$(\widehat{\B{\Theta}}_\lambda,(\widehat{\B{\Theta}}_\lambda)^{-1}) : =  (\widehat{\B{\Theta}},(\widehat{\B{\Theta}})^{-1} )$ 
\end{enumerate} 
\end{algorithm}

\paragraph{Convergence criterion}
The convergence criterion is based upon the relative difference in objective values between two successive iterations (i.e. 
sweeps across all the $p$ rows/columns), being less than a threshold. As described later, the objective value is computed on the `fly', so expensive log-det computations need not be done separately.

\paragraph{Initialization of precision and covariance matrices}
\pine requires as input, initialization for the tuple $(\widetilde{\B\Theta},\widetilde{\B{\Theta}}^{-1})$.
In case no prior choice for the input initialization is available,
we use $\widetilde{\B{\Theta}}^{-1} \leftarrow$ $\mathrm{diag}(s_{11} + \rho,\cdots, s_{pp} +\rho)$. 
Note that the diagonals of $\B{\Theta}^{-1}$ at the KKT optimality conditions for (\ref{crit1}) is precisely the vector
$(s_{11} + \rho,\cdots, s_{pp} +\rho)$.

Often \pine is used for computing a path of solutions to (\ref{crit1}) via warm-starts --- in such a case the tuple $(\widetilde{\B\Theta},\widetilde{\B{\Theta}}^{-1})$ is available as a by-product of the algorithm (see Section \ref{sec:warm}).

\subsection{Important Properties of \pine}\label{sec:props}
We outline some of the important properties of our Algorithm --- which is instrumental in making it  
flexible. For ease of exposition the technical details are relegated to the Supplementary Materials Section~\ref{app:sec:prop}.
\subsubsection{Positive definiteness of precision \& covariance matrices across the iterates}\label{sec:prop:pd}
If the starting matrix $\widetilde{\B{\Theta}} \succ \mathbf{0}$, then every row/column update 
in  Step-2 of Algorithm~\ref{algo:block-gl} 
preserves positive definiteness of the updated matrix. 
For a rigorous proof see Section~\ref{sec:pos-def} (supplementary materials).
\subsubsection{Tracking precision and covariance matrices at every iteration}\label{track-inv}
The function (\ref{margin-diag-2}) that arises while updating the $p^{\mathrm{th}}$ row/column requires knowledge of 
$(\widetilde{\B\Theta}_{11})^{-1}$. 
Of course, it is not desirable to compute the inverse from scratch
for every row/column $i$, with a complexity of $O(p^3)$. 
However, if \emph{both} the current precision and covariance matrices i.e. $(\widetilde{\B\Theta},(\widetilde{\B\Theta})^{-1})$ 
are stored at every iteration then it is quite simple to obtain $(\widetilde{\B\Theta}_{11})^{-1}$ via a 
matrix rank-one update as described in (\ref{update-rank-one-1}). This costs $O(p^2)$ and moreover is amenable to parallelism.
Similarly after every row/column update in $\widehat{\B\Theta}$ its inverse can be obtained via a rank-one update as described in 
(\ref{update-rank-one-2}). 
Details of this implementation involve careful attention to details that are presented in the 
Section~\ref{sec:track} (supplementary materials).  

Tracking both $\B\Theta, \B\Theta^{-1}$ along the iteration provides flexibility to our algorithm in terms of: 
\begin{itemize}
\item We avoid the additional cost of matrix inversion --- $O(p^3)$. 
\item Termination at a given computational budget which
is crucial for large scale problems and often desirable for exploratory analysis. 
Since the operating variable is $\B\Theta$ --- the precision matrix estimate is sparse
\footnote{This is different from the dual optimization problem, where the estimated positive definite precision matrices need not be sparse}.
\item They provide the perfect recipe for warm-starts, when one is interested in computing a path of solutions to (\ref{crit1}) (see
Section \ref{sec:warm}). 
\item It gives a simple but efficient way to evaluate the log-determinant of the precision matrices along iterations, since computing the log-likelihood in large problems is a fairly expensive task. 
\end{itemize}

\subsection{Path Seeking Strategy} \label{sec:warm}
In many real life applications it is desirable to compute a path of solutions $\{\widehat{\B{\Theta}}_\lambda\}_\lambda$ over a grid of 
$\lambda$-values $\lambda_K > \lambda_{K-1} > \ldots > \lambda_1$.  
One method is to compute the 
solutions across different tuning parameter values independently of each other, say on different machines. 
Otherwise, they can be computed serially 
wherein warm-starts/continuation strategies turn out to be very effective \citep{FHT2007}.
In such a case, the estimate at $\lambda_{k}$ i.e.  $(\widehat{\B{\Theta}}_{\lambda_{k}},(\widehat{\B{\Theta}}_{\lambda_k})^{-1})$ is 
taken as an input\footnote{Note that the primal formulation is unconstrained --- so warm-starts do not run into infeasibility problems.} for the Algorithm \ref{algo:block-gl} at $\lambda=\lambda_{k-1}$, for every $k = K, \ldots, 2$.
See Section \ref{sec:expts} for experimental studies showing impressive improvements.

\section{Convergence analysis}\label{sec:conv-analysis}
In this section we will analyze the convergence properties of Algorithm \ref{algo:block-gl}. 
We summarize below the novelty and importance of addressing convergence analysis in this paper:

Firstly, our proposal is not a conventional form of block coordinate descent as described in 
\citet{Tseng01,FHT2007a}, where the partial optimization problem (with the other variables fixed) is completely optimized.
A complete-block coordinate minimization when applied to our problem requires a
full minimization in Step 2 of Algorithm~\ref{algo:block-gl}, over the $i^{\mathrm{th}}$ row/column. 
We differ by replacing this conventional \emph{full} optimization strategy by a \emph{partial} optimization --- namely one pass of coordinate descent as described in Algorithm~\ref{algo:inner-cd}. 

Secondly, our objective function of interest in non-smooth and due to the symmetry of the problem, 
the blocks i.e. the rows and columns have shared elements. Since $\theta_{12}=\theta_{21}$ the value gets updated twice --- once at 
row/column=1 and the other at row/column = 2. 
Conventional forms of block coordinate minimization theorems \citep{Tseng01} for non-smooth functions demand separability (in blocks) --- so they 
do not apply directly. \citet{WEN_GOLDFARB_Scheinberg_2009} highlight this issue of overlapping entries and provide a proof of convergence. 
The work of \citet{WEN_GOLDFARB_Scheinberg_2009} considers smooth objectives --- hence the results do not directly apply to our problem.  

Note that by construction the sequence of precision matrices produced by Algorithm~\ref{algo:block-gl} results in a monotone decreasing sequence of objective values. Even if the objective values converge (which is true if they are bounded from below), it is not necessary that the point of convergence corresponds to the minimum of the problem (\ref{crit1}) --- 
the sequence of precision matrices need not converge either (see \citet{Tseng01} for discussions). We address these issues and show that that the precision matrix estimates \emph{converge} to the minimum under the mild assumption $\lambda >0$. 
Convergence holds for $\lambda =0$ under the extra assumption that $\s \succ \M{0}$.

The convergence analysis we present here is to the best of our knowledge novel.

We start with an important Lemma appearing in \citet{Lu:10}[Proposition 3.1]:
\begin{lemma}
\label{lem-covsel}
For every $\lambda>0$, problem (\ref{crit1}) has a \emph{unique} minimizer ---  $\B{\Theta}^*_\lambda$, which is 
(strictly) positive-definite and satisfies: 
$$ \beta I_{p \times p} \geq  \B{\Theta}^*_{\lambda}  \geq \alpha I_{p \times p} $$ 
for scalars $\alpha, \beta$, depending upon $\s,\lambda, p$ with $\infty >\beta \geq \alpha >0$.
\end{lemma}

We are now ready to state the main theorem establishing the convergence of Algorithm~\ref{algo:block-gl}
\begin{theorem}[Asymptotic Convergence of \pine]
\label{thm:conv}
Take $\lambda>0$. 
Let $\B{\Theta}_{k}$ be the estimate of the precision 
matrix obtained at iteration $k$ i.e. on completion of Step 3 of Algorithm~\ref{algo:block-gl}. 
Then the following hold true:
\begin{enumerate}
\item[(a)] The sequence of objective values is monotone decreasing :
\begin{equation}\label{monotone}
g(\B{\Theta}_{k+1}) \leq g(\B{\Theta}_k), \forall k \geq 1. 
\end{equation}
The sequence $\{g(\B{\Theta}_{k})\}_k$ converges to the optimal solution of problem (\ref{crit1}).
\item[(b)] The iterates $\B{\Theta}_k \succ \mathbf{0}, \forall k$ and the sequence \emph{converges} to $\B{\Theta}_{\infty}$ --- the unique solution to (\ref{crit1}).
\end{enumerate}
\begin{proof}
The proof, which is rather detailed and technical is provided in the Appendix, Section \ref{proof-conv}.
\end{proof}
\end{theorem} 
\section{Some variants of \pine}
This section discusses some variations to our proposal \pine --- leading to two important variants.

A variant of Algorithm~\ref{algo:inner-cd} is having a counter for the number iterations for Step~\ref{item-12}, say $T_o$.
Our proposal of \emph{inexact} minimization and for that matter the overall complexity analysis demands 
$T_o=O(1)$ i.e. $T_o \ll p$.  
In our numerical experiments we found $T_o \leq 2$ to be quite practical.

If $T_o$ is taken to be arbitrarily large, we get the conventional form of cyclical coordinate descent used for $\ell_1$ 
regularized quadratic programs (QP) \citep{FHT2007}. 
The magnitude of $T_o$ depends upon the accuracy of the solution for the $\ell_1$ regularized QP. 
In general, for a high accuracy solution, this can be arbitrarily large. If $T_o = O(p)$ this leads to a $O(p^3)$ complexity of Algorithm \ref{algo:inner-cd}.  
See Section \ref{sec:complexity} for details. We call this variant 
\MT{P}rimal \MT{E}xact Minimization for \MT{G}raphical \MT{L}asso i.e. \pex --- this is the more conventional form of block coordinate descent applied on the problem (\ref{crit1}). 

We now proceed to discuss another simple but important variant of our algorithm \pine, namely a `growing' strategy 
--- which we call \MT{P}rimal \MT{GR}owth for \MT{G}raphical \MT{L}asso i.e. \pgr.
\subsection{Primal Growth for Graphical Lasso (\pgr)}\label{sec:grow} 
Given an initial working dimension $p_0$ (typically $p_0=1$) and estimates of the precision and the covariance matrix
$(\widetilde{\B{\Theta}}_{p_0 \times p_0},(\widetilde{\B{\Theta}}_{p_0 \times p_0})^{-1})$, Algorithm~\ref{algo:grow}  
(Supplementary Materials, Section~\ref{sec:grow:algo}) describes the task of obtaining the 
solution to (\ref{crit1})  (with $\s$ having dimension $p \times p$). 
The main idea is to perform an initial 
forward pass by \textit{incrementally} appending rows/columns and operating   
Step 2 of Algorithm~\ref{algo:block-gl} on the just added row/ column. Once the growing matrix is saturated to have
 $p$ rows/columns --- we make further passes through the $p$ rows/columns, via Step-2 of Algorithm~\ref{algo:block-gl}, till convergence.
Since the `growing' stage of the algorithm performs mainly cheap computations, it helps in  
providing pretty accurate warm-starts/ initializations $\widetilde{\B{\Theta}}, (\widetilde{\B{\Theta}})^{-1}$ to \pine, within a very short amount of time. See also results in Section~\ref{sec:expts}.
When the task is to solve (\ref{crit1}) for a single value of $\lambda$, the method \pgr often turns out to be quite competitive with 
\pine. 

\begin{remark}
The convergence of \pgr is straightforward. Firstly, it is not hard to see that the iterates maintain positive definiteness of the precision and its inverse and furthermore, since \pine comes into action after one full-sweep of incrementally growing rows/columns, the convergence analysis for \pine carries over.
\end{remark}
\section{Computational Complexity}\label{sec:complexity}
Here we describe the computational complexities of our proposed algorithms \pine, \pex and \pgr.
We provide a summarized report here, the details are available in Appendix, Section \ref{comp-complex-append}.\\
Cost of \pine :  Every row/column update requires $O(p^2)$, and for a full sweep across $p$ rows/columns --- this is $O(p^3)$. For $\kappa$ full sweeps across $p$ rows/columns this is $O(\kappa p^3)$, typically convergence occurs within $\kappa \approx$ 2 - 10. See Section \ref{sec:pine}.

Cost of \pex : For every row/column the cost at the worst is $O(p^3)$. For a total of $\kappa'$ ($\approx$ 4-10) sweeps across all rows/columns the cost is $O(\kappa'p^4)$. The cost may reduce to $O(p^3)$ in case $\lambda$ is quite large. See Section \ref{sec:pex}.

Cost of \pgr : The cost here is $O(p^3)$ as in \pine --- but the constants are generally better than that of \pine. See Section \ref{sec:pgr}.

\section{Related work}\label{sec:related}
In this section we briefly describe some of the state-of-the art algorithms for criterion (\ref{crit1}), their computational complexities 
and their fundamental differences with our proposal(s).

The block coordinate proposals of \citet{BGA2008,FHT2007a} are related to our proposal --- they solve the dual of the problem (\ref{crit1}), which is given by:
\begin{equation} \label{covsel-dual} 
\max_{\|\mathbf{V}\|_\infty \leq \lambda}  -\log\det(\s + \mathbf{V}) - p.  
\end{equation} 
By strong duality the optimal solution of problem (\ref{covsel-dual}) and (\ref{crit1}) are the same,
the primal-dual relationship being $(\B\Theta)^{-1}= \s + \mathbf{V}$. 
(\ref{covsel-dual}) operates on the covariance matrix whereas the primal problem (\ref{crit1}) operates on the precision matrix.
As pointed out earlier, there is significant difference in pursuing the primal approach versus the dual. 
Often in real-life applications (as is the case in a principal motivating application for this paper \citep{AZM-11}) 
one desires an approximate solution since it gives a fairly good statistical estimate for the main statistical estimation problem. 
An approximate solution in the dual space need not translate to one of similar accuracy in the primal space according to criterion (\ref{crit1}). Further the dual approach does not produce sparse precision matrices ---
if $\widehat{\M{V}}$ solves (\ref{covsel-dual}), then the precision matrix 
$\widehat{\B\Theta}=  (\s + \widehat{\M{V}})^{-1}$  is \emph{not} sparse unless (\ref{covsel-dual}) is solved till high 
tolerance ($10^{-8}$ --$10^{-10}$) .
Ad-hoc thresholding strategies / post-processing strategies can be used to sparsify $\widehat{\B\Theta}$ --- but positive definiteness is not guaranteed. 

The block coordinate maximization of \citet{BGA2008} on (\ref{covsel-dual}), requires \emph{solving}
 a box-constrained QP  completely --- with 
cost $O(p^3)$. The \GL (Graphical Lasso) Algorithm of \citet{FHT2007a} minimizes the dual of the same box-constrained QP --- an $\ell_1$ regularized QP via cyclical coordinate descent. In the worst case this can be $O(p^3)$, 
in case the solutions are very sparse this is  $O(p^2)$. 
Inexact minimization strategies for \GL do not guarantee convergence. 
\GL need not produce a sparse and positive definite precision matrix unless it converges to a high accuracy.

To summarize, both the block-coordinate proposals of \citet{BGA2008,FHT2007a} have a worst case cost $O(p^4)$ ---  
the latter can improve to 
$O(p^3)$ if $\lambda$ is very large.

The gradient based algorithm of \citet{BGA2008} inspired by \citet{nest_05} has a per-iteration complexity 
$O(p^3)$ and overall complexity $O(\frac{p^{4.5}}{\epsilon})$ (for an $\epsilon >0$ accurate solution).

Another very efficient gradient-based algorithm is \textsc{smacs} proposed in \citet{Lu:10}, which also solves the dual formulation. This has per iteration complexity $O(p^3)$ (due to expensive matrix operations like eigen-decompositions, matrix inversions)
and an overall complexity of $O(\frac{p^4}{\sqrt{\epsilon}})$. 

The number of iterations taken by \GL \citep{FHT2007a,BGA2008} and \pine (and its variants like \pex, \pgr) i.e. full sweeps across all rows and columns are relatively small in most examples --- of the order of 4-10. For a solution of similar accuracy, the number of gradient iterations for \sSMACS is often of the order of hundreds (or even more than a thousand) for problems of size 1000/1500. 

It appears that most existing algorithms for solving the sparse covariance selection problem have
a complexity of $O(p^4)$ or possibly larger, depending upon the algorithm used and the desired accuracy of the solution --- making computations for (\ref{crit1}) almost impractical for values of $p$ much larger than 1000/1500.

In contrast,  every row/column update of \pine is $O(p^2)$ --- overall for $\kappa$ sweeps across all rows/columns this is $O(\kappa p^3)$, where $\kappa$ denotes the total number of sweeps across all the rows/columns (See Section \ref{sec:complexity}). 
This is clearly an order of magnitude improvement over existing algorithms and is further substantiated by our experiments.

\section{Experimental Studies : synthetic examples}\label{sec:expts}
This section provides a comparison of our proposed fitting methods with some state-of-the-art algorithms 
for the optimization problem (\ref{crit1}). 

We use our main proposal \pine, its close cousin \pgr, and the variant \pex for comparisons. 

Among the existing algorithms, 
\citet{Lu:10} was observed to be better than the proposal of \citet{Lu:09}, so we used the former for our comparisons.
\citet{Scheinberg_Ma_Goldfarb_2010,Yuan_2009,boyd-admm} describe algorithms based on the Alternating Direction Methods of Multipliers --- among them the algorithm of \citet{boyd-admm} was publicly available at Stephen Boyd's website. 
We experimented with this algorithm, but found it to be slower than \GL, so we did not include it for our comparisons.

We thus compared our proposals with two very efficient algorithms : 
\begin{description}
\item [\sglasso:] The algorithm of \citet{FHT2007a}. We used the MATLAB wrapper around their Fortran code --- available at 
\url{http://www-stat.stanford.edu/~tibs/glasso/index.html}. 
\item[\sSMACS:] denotes the algorithm of \citet{Lu:10}. We used the MATLAB
implementation \verb$smooth_covsel$ 
available at \url{http://people.math.sfu.ca/~zhaosong/Codes/SMOOTH_COVSEL/}.
\end{description}
Note that the default convergence criteria for the above two algorithms
are different --- \GL checks the successive changes in the diagonals
of the precision matrix, whereas \sSMACS\ relies on duality gap.
Moreover \GL and \sSMACS\ operate on the dual, whereas our proposals
\pine, \pgr, \pex operate on the primal.  
Since, solving (\ref{crit1}) is the main goal, to make comparisons
fair, we compared the primal likelihoods of the estimates produced by
the algorithms.

A relatively weak convergence criterion on the dual 
is typically quite far off from a sparse and positive definite precision matrix.
The \GL algorithm
tracks a precision matrix $\B\Theta$ and covariance matrix $\M W$
along the iterations but $(\B\Theta)^{-1} \neq \M W$ and the
discrepancy can be quite large before the algorithm converges to a
high accuracy in the dual space. Furthermore, even if the estimated precision matrix
(prior to convergence) is sparse it need not be positive definite.
\sSMACS\ produces estimates of precision matrices $\B\Theta$ along the
iterations --- though they are positive definite, they are dense. Arbitrary thresholding (to zero) of
the smaller entries may destroy positive definiteness of the matrix.

 Our proposal on the other hand at every iteration tracks 
the precision matrix (which is both sparse
 and positive definite) and its (exact) inverse. 

In order to make the primal and dual problems
comparable we consider the times taken by the algorithms to converge
till a relatively high tolerance i.e. TOL=$10^{-5}$, where 
\begin{equation} \label{eq-tol}
\text{Convergence Test Criterion:} \quad
\frac{( g(\B\Theta_k) - \widehat{g(\B\Theta_{*})} ) } {|\widehat{g(\B\Theta_{*})}|} < \text{TOL}.
\end{equation} 
Here $\B\Theta_k$ is the estimate of the precision matrix produced by the respective algorithm at the end of iteration $k$, 
and  $\widehat{g(\B\Theta_{*})}$ is the estimate of the minimum of (\ref{crit1}) obtained by taking 
the minimum over different algorithms after running them for a large number of iterations\footnote{In our examples we ran 
\pine, \pex, \pgr (each) for 30 iterations. They were enough to give solutions till an accuracy of $10^{-8}$.}.

All of our computations are done in MATLAB 7.11.0 on a 3.3 GhZ Intel Xeon processor, with single-computational-thread computations enabled.
Our codes are written in MATLAB and C \footnote{The C code was generated via the embedded-matlab function { \tt emlc}, an automated 
C code generator in Real Time Workshop in MATLAB.}. \GL is written entirely in Fortran. 
\sSMACS is written in MATLAB --- we don't think this puts \sSMACS at a (timing) disadvantage, since the major computations are matrix operations which are very well optimized in MATLAB.
\subsection{Algorithm Timings}\label{sec:synthetic}
The simulation examples we used were inspired by \citet{Lu:10}.  
The population precision matrix $\Sigma^{-1}_{p \times p}$ having approximately $0.01$ proportion of non-zeros is generated as follows.
We generate a matrix $A_{p \times p}$ with entries in $\{-1,0,1\}$, with  
proportion of non-zeros $0.01$. The $-1$ and $1$ occur with equal probability. 
$A$ is symmetrized via 
$A \leftarrow 0.5 \cdot (A + A')$. All the eigen-values of $A$ are lifted up by adding a scalar multiple of a 
identity matrix : $\Sigma^{-1} \leftarrow A + \tau I_{p \times p}$ such that the minimal 
eigen-value of $\Sigma^{-1}$ is one. The (population) covariance matrix is taken to be $\Sigma$.
We then generated $x_i \sim MVN(0,\Sigma),i = 1,\ldots, N$. The sample correlation matrix was taken as $\s$.
 
We considered a battery of examples with varying $N,p$:\\
(a)  $N \in \{500,1000,2000\}$ for $p = 1000$ and (b) $N \in \{2000,3000,4000\}$ for $p = 1500$

\begin{table}[htpb!]
  \begin{tabular}{| l | c |  c c c c c | }     \hline
\multirow{2}{*}{p / N }         & \% of &   \multicolumn{5}{|c|}{Algorithm Times (sec)}   \\
   &   nnz & \pgr & \pex & \pine & \GL & \sSMACS  \\ \hline 
\multirow{3}{*}{$1\times 10^3$ / $2\times 10^3$}  &  92  & \bf 48.8 & 140.4 & 119.4 & 143.8 &  308.5 \\ 
 & 78 & 143.3 &  130.5 &  \bf 94.7 &  151.8 &  288.6 \\ 
                       & 46 & 149.7 & 167.4  & \bf 108.6  &220.1  &217.7    \\ \hline
\multirow{3}{*}{$1\times 10^3$ / $1\times 10^3$} & 85 & 79.1 & 143.1 &  \bf 66.9 &  130.4  & 398.7 \\ 
 &71 &  143.2 &   150.6 & \bf 69.0 & 160.5 & 408.2 \\
&49 & \bf 132.6 & 225.3 & 187.2 & 295.7 & 464.2 \\ \hline
\multirow{3}{*}{$1\times 10^3$ / $0.5\times 10^3$}  &  91 &  82.4 & 93.8 &\bf 66.1 & 94.4  & --- \\
 & 73 & \bf 81.5 &   123.2 &  119.5 & 179.7 & --- \\
& 48 &  354.8 & 382.4 &  \bf   340.2 & 544.5  & --- \\ \hline
 \multirow{3}{*}{$1.5\times 10^3$ / $4\times 10^3$} & 86 &  223.2 & 258.7 & \bf 186.3 & 571.1 & 2310.4 \\
& 72 &   221.8 &  353.4 &\bf 186.5 & 577.8 & 1534.5 \\
& 47 &  \bf 401.8 & 656.1 & 488.4 &  851.8 &  1062.8 \\ \hline
\multirow{3}{*}{$1.5\times 10^3$ / $3\times 10^3$} & 86 & 212.6 & 228.9 & \bf 177.4 & 533.3 & 1736.3 \\
&72 & 221.3   & 256.4 &  \bf 186.0 &  573.7 &  2017.2 \\ 
& 48 & 525.6 & 675.8 & \bf 494.2 & 880.4 & 1521.4 \\ \hline
\multirow{3}{*}{$1.5\times 10^3$ / $2\times 10^3$} & 85 & 283.3 & 364.8 & \bf 222.7 & 566.5 &  1759.3 \\
& 72 &  222.6 &   258  & \bf 186.3 &  574.2 &   2246.9 \\
& 40 &   757.8 &   1019.7 &\bf 706.6 & 1186.6 &  1780.3 \\\hline
 \end{tabular} 
\caption{Table showing the times in seconds till convergence to a tolerance of TOL$=10^{-5}$ (\ref{eq-tol}), for different algorithms for different problem set-ups. 
For every combination of $(p,N)$ three different
$r\lambda$ values are considered --- as indicated by the \% non-zeroes in the final solution for \pex. 
All algorithms are warm-started at their default starting values. The `---'
corresponding to \sSMACS indicates that the algorithm did not converge for this example with $N < p$.} \label{tab:one}
\end{table}

Table \ref{tab:one} summarizes the timing results (in seconds) for the examples above for different algorithms. 
The timings are shown for different $\lambda$ values --- algorithms are cold-started at their default starting points.  
We see that \pgr, \pine are always the winners and often by multiplicative factors. 
\pine turns out to be the clear winner overall.
\pex turns out to be slower than \pine --- which supports our crucial idea of inexact minimization in the inner-blocks and also supports our complexity analysis. As expected, the timings for the block coordinate algorithms deteriorate for smaller values of $\lambda$. 
For dense problems (which are arguably harder problems for the primal formulation), \pgr consistently performs quite well.
\pex and \GL often perform similarly. \sSMACS tends to be quite slow for larger problems, when compared to the block coordinate counterparts. We found \sSMACS to be quite competitive for very small values of $\lambda$ --- but these (almost) unregularized solutions are not much statistically meaningful unless $n > p$. \sSMACS faced problems with convergence for $n < p$ situations where $\s$ was low-rank.

These results demonstrate the impressive comparative performances of \pine and \pgr  compared to current 
state-of-the-art methods --- making it probably a method of choice in scenarios where it is possible to run the fitting
algorithms till a high tolerance.
As mentioned earlier, the primal formulation is particularly suited for this task of delivering solutions with lower tolerances.
It operates on the primal (\ref{crit1}) delivers a sparse and positive definite precision matrix and its exact inverse.  
Table \ref{tab:three} (Supplementary Material, Section \ref{sec:low-high-accuracy}) 
shows average timings in seconds across a grid of ten $\lambda$ values with varying degrees of accuracy.
\pine, \pex and \pgr return lower accuracy solutions to 
(\ref{crit1}) --- in times which are much less than that taken to obtain higher accuracy solutions.
The gains are rather substantial given the limited scope of the dual optimization problems in the `early stopping' paradigm.

The next section compares different algorithms as `path-algorithms'. Path-based-algorithms and algorithms operating on a single value of $\lambda$ are quite different performance-wise. An algorithm that tends to be very fast as a path algorithm need not be as good at a single value of $\lambda$. This is because, a good warm-start improves the convergence-rate of the algorithm.
Similarly, an algorithm that is very good at a single value of $\lambda$, may not benefit much from warm-starts (a typical example being interior point methods). This is why we compare our proposals in both scenarios.   
\subsection{Tracing out a path of solutions}\label{sec:path-algo}
Continuing with Section~\ref{sec:warm}, we see what happens to the rate of convergence of \pine in 
presence of warm-starts. 
Note that \sSMACS and \pgr do not allow for warm-starts but \GL and \pex do.
The data is the same as used in the previous section. We took a grid of ten $\lambda$ values as follows:
the off-diagonal entries of the sample covariance matrix $\s$ were sorted as per their absolute values and ten $\lambda$ values were chosen from the entire range (along equi-spaced percentiles of the absolute values in $\s$) --- 
the largest $\lambda$ value being $\eta_{\max}q$ and the smallest was 
$\eta_{\min} \cdot q$, where $ q:= \max_{i > j} |s_{ij}|$. All algorithms were run till a tolerance of $10^{-5}$. Table \ref{tab:3} summarizes the results.
\begin{table}[htpb!]
  \begin{tabular}{| l | c | c | c c c  | c |}     \hline
\multirow{2}{*}{p / N }   & \multirow{1}{*}{$\eta_{\max}$/  $\eta_{\min}$} &average \%& \multicolumn{3}{c|}{Algorithm Times (sec)}   &speed-up\\
                           &       ($\times 10^{-2}$)          &                          of nnz  &  \pine & \GL  & \sSMACS &  (\pine)\\ \hline 
\multirow{1}{*}{$1\times 10^3$ / $2\times 10^3$} &  16/ 0.64 & 61.3 &   \bf 77.27   & 144.10   & 250.43  & 1.56 \\
\multirow{1}{*}{$1\times 10^3$ / $1\times 10^3$} &  21/ 0.83 & 62.8  &    \bf 116.38    & 202.73& 412.14 & 1.52 \\
\multirow{1}{*}{$1\times 10^3$ / $0.5\times 10^3$} & 28 / 2 &  66.7 &   \bf 105.52  &  315.17 &    --- & 1.89 \\ 
\multirow{1}{*}{$1.5\times 10^3$ / $4\times 10^3$} &  13.1 / 0.4 & 62.4  &  \bf 260.54 & 579.52 & 145.35 & 1.28\\ 
\multirow{1}{*}{$1.5\times 10^3$ / $3\times 10^3$} &  14.3 / 0.53 &  62.8 & \bf 280.19 &  613.67 &  1631.6 & 1.23\\ 
\multirow{1}{*}{$1.5\times 10^3$ / $2\times 10^3$} &  16.0 / 0.61& 63.1 &   \bf 267.03 &  697.65 &  1892.1 & 1.53\\ \hline
\end{tabular}  
\caption{Table showing the comparative timings (in seconds) of the three algorithms \pine , \GL and \sSMACS across a grid of ten $\lambda$ values. Times are \emph{averaged} across the ten $\lambda$ values. The averaged \% of non-zeros in the final solution across the different $\lambda$ values along with the limits of the $\lambda$ values are also shown. The last column shows the speed-up factor for \pine using warm-starts over the time spent to compute the solutions of the same accuracy without using warm-starts.} \label{tab:3}
\end{table}

As the column on speed-up factor shows, the path algorithm of \pine is much faster than 
obtaining the solutions at the same values of $\lambda$ without warm-starts.
\pine continues to perform very well when compared to the path algorithm \GL.

\section{Real Application: Learning Precision graphs for Movie-Movie Similarities}
\textbf{ MovieLens Data Set:} We study an application of the inverse
covariance estimation method on a dataset obtained from a movie
recommendation problem. We use the benchmark MovieLens-1M dataset
available at \url{http://www.grouplens.org/node/12}, which consists of
1M explicit movie ratings by 6,040 users to 3,706 movies. 
The explicit
ratings are on a 5-point ordinal scale, higher values indicative of
stronger user preference for the movie. 
The statistical problem that has received
considerable attention in the literature is that of predicting
explicit ratings for missing user-movie pairs.  Past
studies~\citep{ruslan} have shown that using movie-movie similarities
based on ``who-rated-what'' information is strongly correlated with
how users explicitly rate movies. Thus using such information as user covariates
helps in improving predictions for explicit ratings. Other than using it as covariates,
one can also derive a movie graph where edge weights represents movie similarities that
are based on global ``who-rated-what'' matrix. Imposing sparsity on such a graph is
attractive since it is intuitive that a movie is generally related to only a few other movies. 
This can be achieved through {\OMOD}. Such a graph provides a neighborhood structure
that can also help directly in better predicting explicit ratings. For instance, in predicting
explicit rating $r_{ij}$ by user $i$ on movie $j$, one can use a weighted average of ratings by the user
in the neighborhood of $j$ derived from the movie-movie graph. 
Such neighborhood information can also
be used as a graph Laplacian to obtain better regularization of user factors in matrix
factorization model as shown in \citet{lurecsys09}. 

Other than providing useful information to predict explicit ratings,
we note that using who-rated-what information also provides
information to study the relationships among movies based on user
ratings. We focus on such an exploratory analysis here but note that
the output can also be used for prediction problems following
strategies discussed above. A complete exploration of such strategies
for prediction purposes in movie recommender applications is involved and
beyond the scope of this paper.

We define the sample movie-movie
similarity matrix as follows: for a movie $j$,  
$\B{x}_j$ is the binary indicator vector denoting users who rated movie $j$. The 
similarity between movie $j$ and $k$ is defined as 
$s_{jk} = \frac{\B{x}_j'\B{x}_k}{\sqrt{\sum\limits x_{j,l}\sum\limits x_{k,l}}}.$ 
The movie-movie similarity matrix $\M{S}$ thus obtained is 
symmetric and postive semi-definite.
\subsection{Timing comparisons}\label{sec:timemlens}
As noted earlier,
for this application we use criterion (\ref{crit1}) in a non-parametric fashion, where our intention is to learn the 
connectivity matrix corresponding to the sparse inverse covariance matrix. 
We will first show timing comparions of our method \pine (the path-seeking version) along with \GL and \sSMACS. 
We see that \pine is the only method that delivers solutions within a reasonable amount of time --- the estimated 
precision matrices are used to learn the connectivity structure among the items. 

We ran \GL for nine $\lambda$ values -- which were equi-spaced quantiles between the 8-th and 75-th percentile 
of the entries $\{|s_{ij}|\}_{i > j}$ --- this range also covers 
estimated precision matrices that are quite dense.

The path versions of \pine, \GL and \sSMACS were used to obtain solutions on the chosen grid of $\lambda$ values. 
The timings are summarized below:
\begin{itemize}
\item \pine produced a path of solutions across the nine $\lambda$ values using warm-starts in a \emph{total} of 6.722 hours. 
\item The path version of \GL on the other hand, could not complete the same task of computing solutions to (\ref{crit1}) on the same set of nine $\lambda$ values, within two full days.
\item  We also tried to use \sSMACS for this problem, but it took more than 14 hours to compute the solution corresponding to a single value of $\lambda$.
\end{itemize}
The timing advantages should not come as a surprise given the computational complexity of \pine is order(s) of magnitude better than its competitors. The performance gap becomes more prominent with increasing dimensions --- traces of evidence were  
observed in Section~\ref{sec:expts}.
\subsection{Description of the Results}\label{sec:realmlens}
Figure~\ref{fig-spy} (Section~\ref{sec:spy} of Supplementary Materials) displays the 
nature of the edge-structures and how they evolve across varying strengths of the shrinkage 
parameter $\lambda$ for the \emph{whole} precision-graphs produced by \pine.

For a fixed precision matrix, 
a natural sub-set of `interesting' edges among the all-possible $p(p-1)/2$ edges are the ones corresponding to the 
top $K$ absolute values of partial correlation coefficients. 
The nodes corresponding to these $K$ edges and the edges of the precision graph restricted to them form a sub-graph of the $p\times p$ precision graph. 
We summed the absolute values of the off-diagonal entries of the precision matrices across the different $\lambda$ values. The (averaged) 
$p(p-1)/2$ values were ordered and the top ten entries were chosen. These represent the partial correlations having 
the maximal strength (on average) across the different $\lambda$ values taken. 
There were 20 vertices corresponding to these top ten partial correlation coefficients. 
Figure \ref{fig:subgraphs} shows the sub-graphs of the movie-movie precision graphs restricted to the selected 20 movies, across different $\lambda$ values. The movie to ID mappings are in Table \ref{tab:movies}, in Section \ref{sec:movie-map} at 
Supplementary Materials Section.
The edges in these subgraphs provide some interesting insights. For instance, consider the sub-graph corresponding to the largest $\lambda$ (highest sparsity). Part of strong connectivity among movies 0,1,2,3 is expected since 0,1 and 2,3 are sequels. It is interesting to see there is a connection between 0 and 3,
both of which are Sci-fi movies related to aliens. Other connections also reveal interesting patterns, these can be investigated using the IMDB movie database.

\begin{figure}[htpb!]
\begin{center}
$\begin{array}{ccc}
\includegraphics[angle=270,totalheight=.4\textheight,width=.32\textwidth]{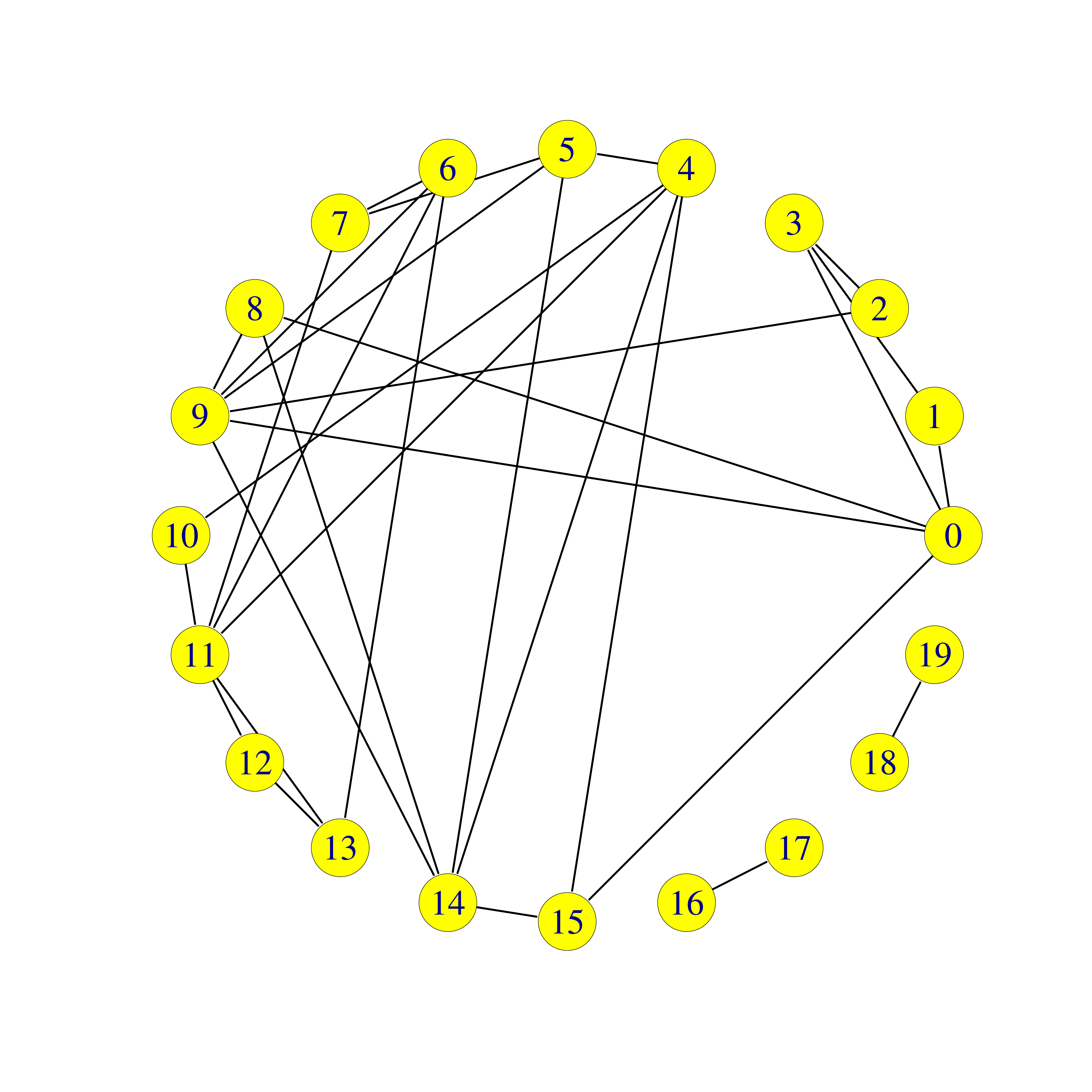} &
	\includegraphics[angle=270,totalheight=.4\textheight,width=.32\textwidth]{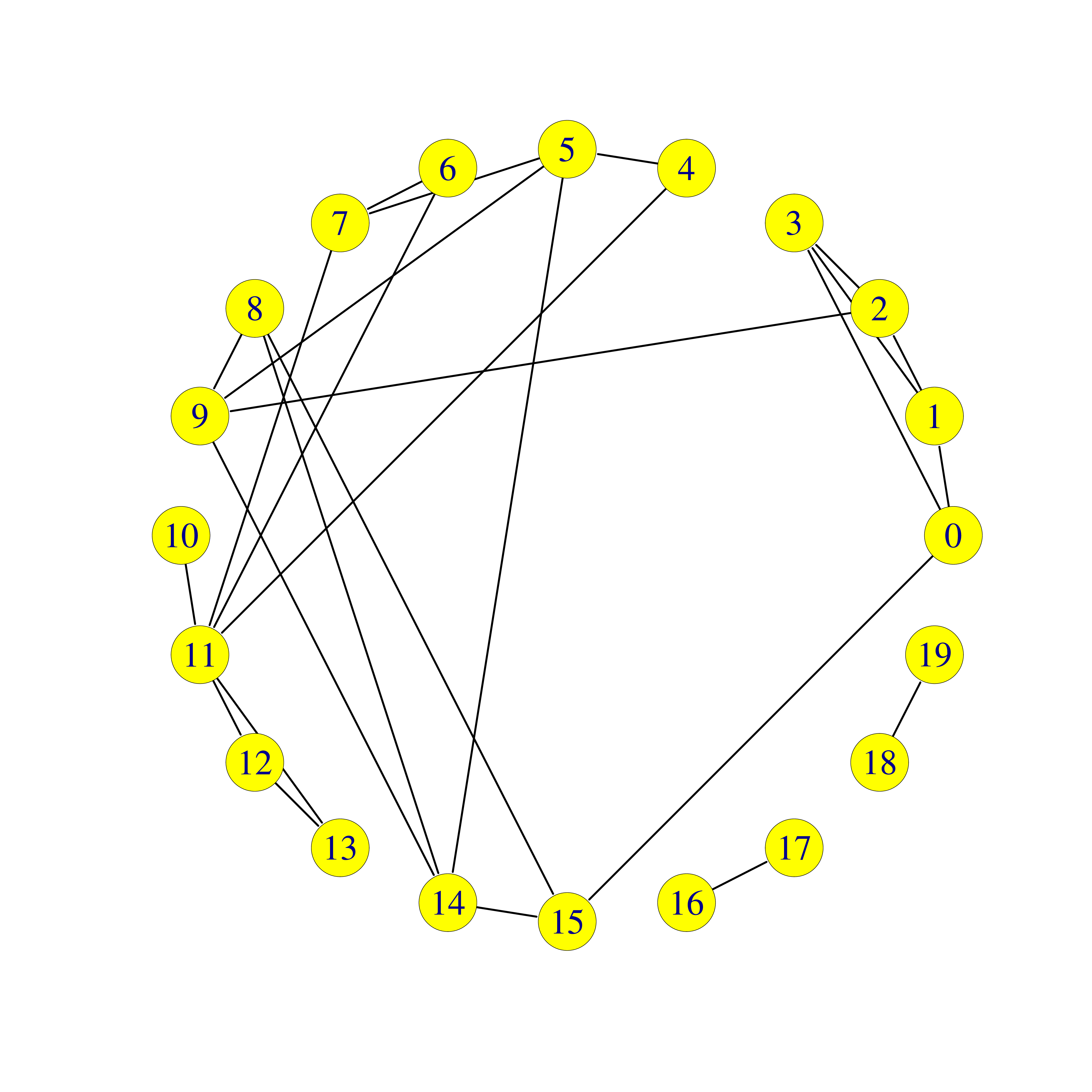} &
\includegraphics[angle=270,totalheight=.4\textheight,width=.32\textwidth]{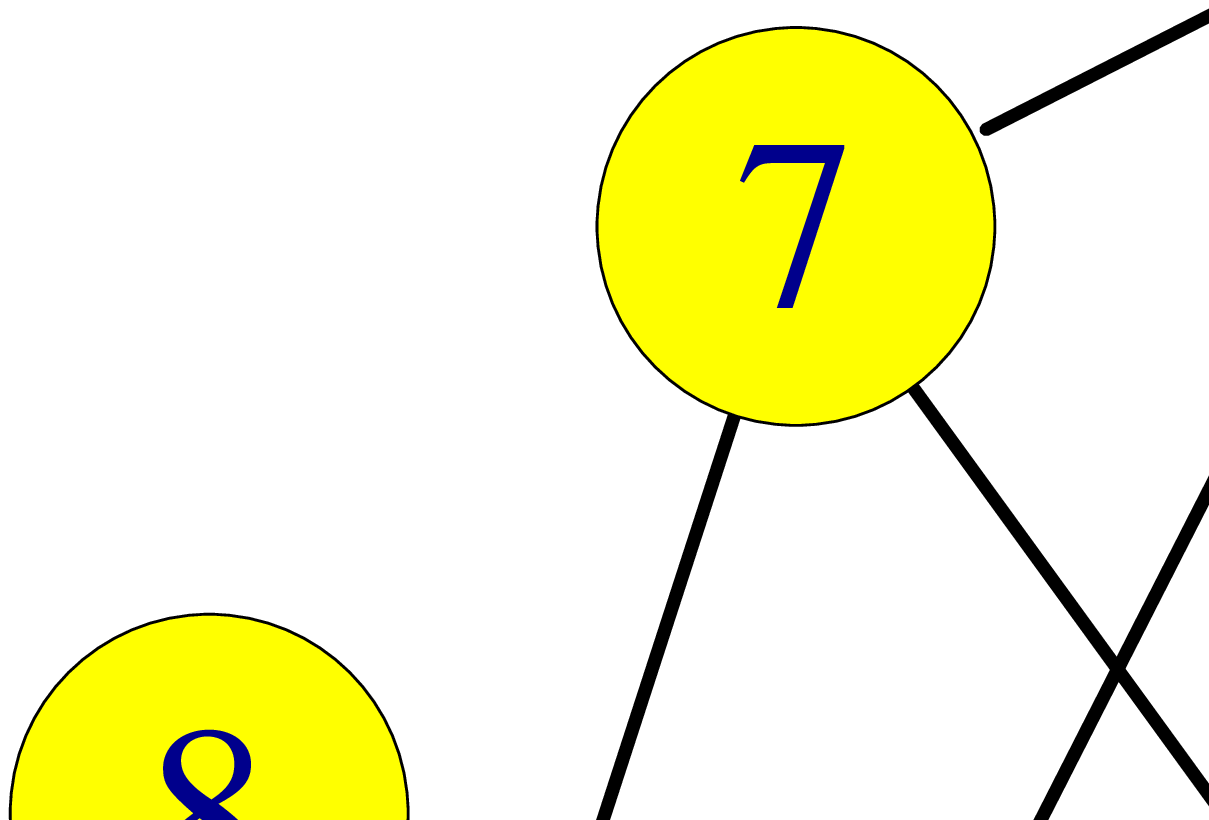}  \\
\mbox{\small{\% non-zeros 90.4}} & \mbox{ \small{\% non-zeros 95.0}} &  \mbox{\small{\% non-zeros 97.8}} \\
\includegraphics[angle=270,totalheight=.4\textheight,width=.32\textwidth]{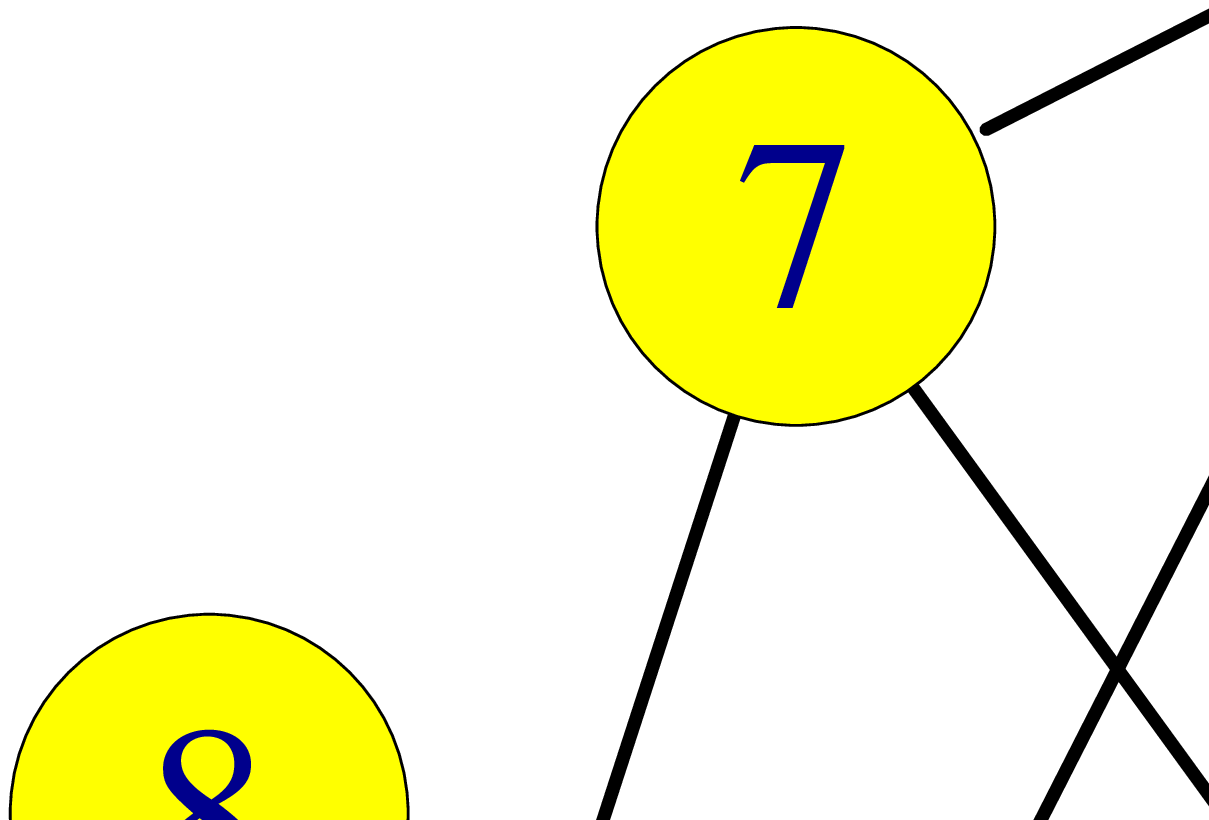} &
\includegraphics[angle=270,totalheight=.4\textheight,width=.32\textwidth]{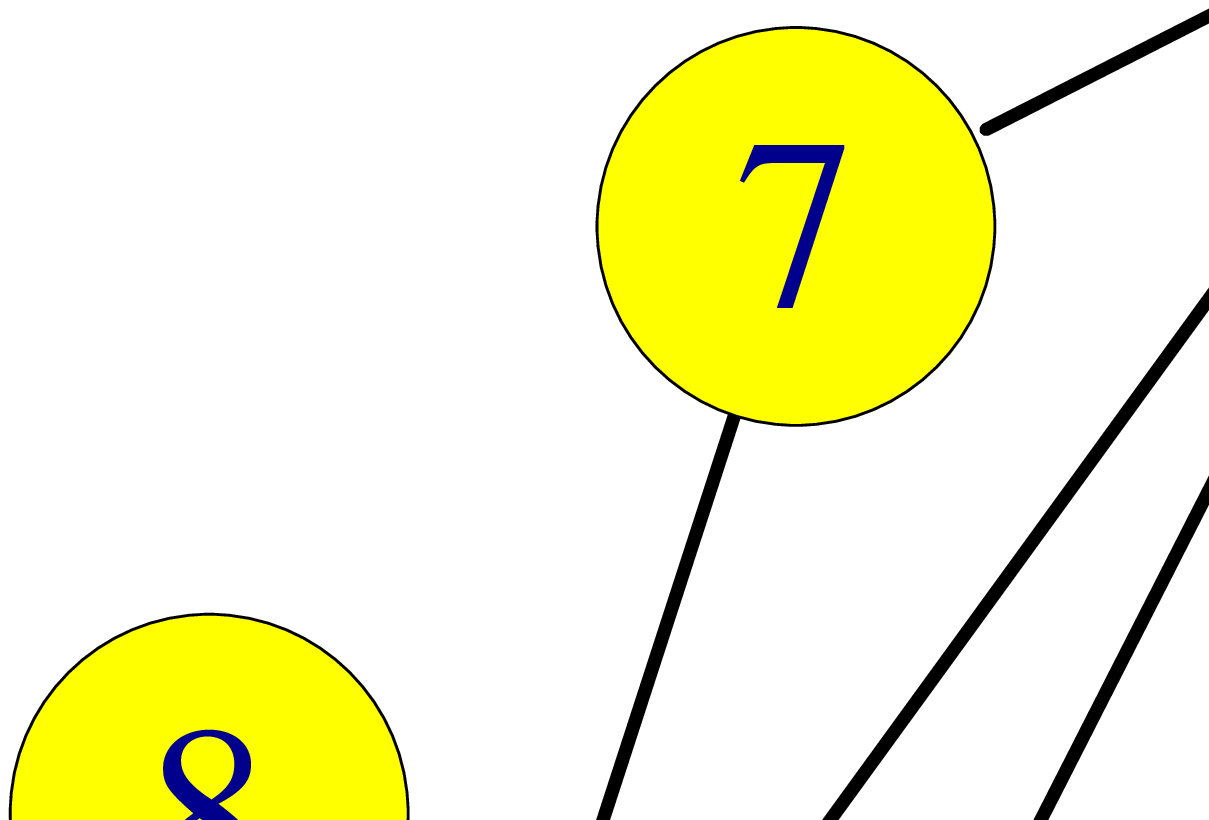} &
\includegraphics[angle=270,totalheight=.4\textheight,width=.32\textwidth]{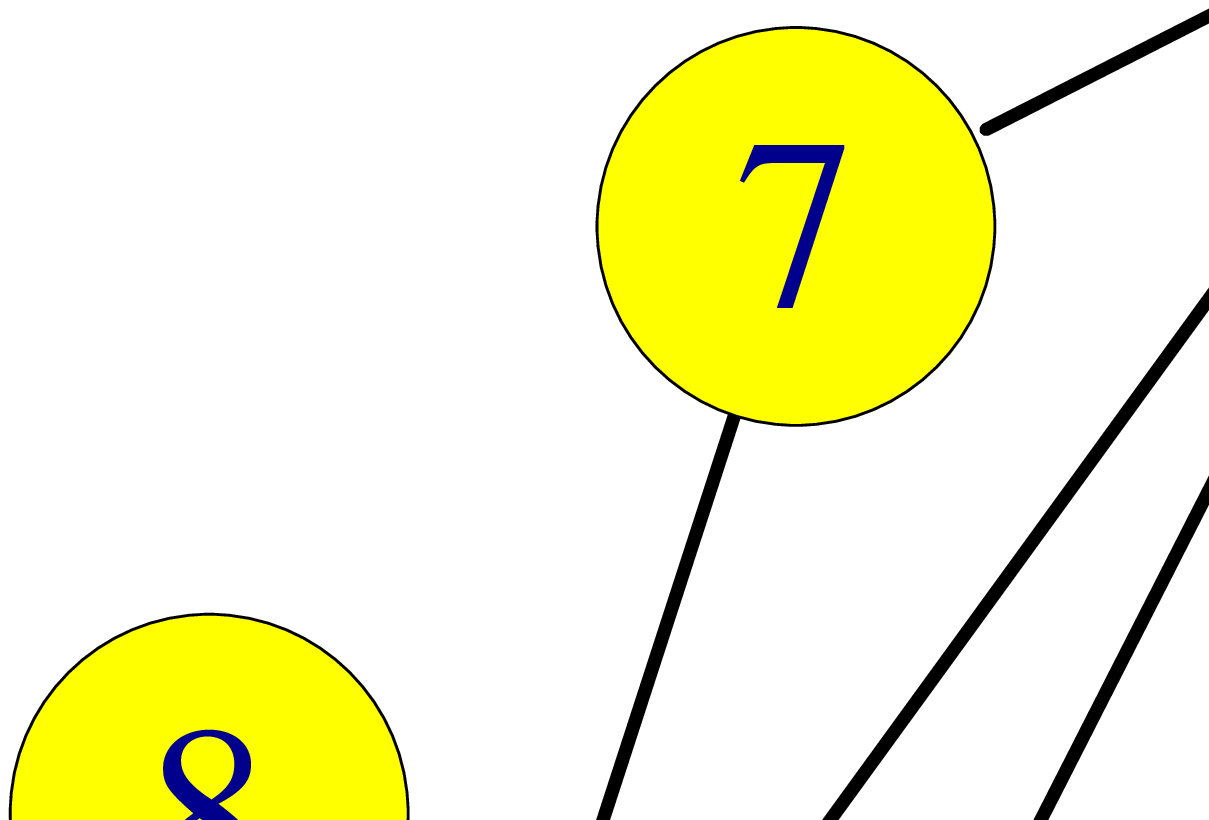} \\
\mbox{\small{\% non-zeros 98.3}} & \mbox{ \small{\% non-zeros 98.8}} &  \mbox{\small{\% non-zeros 98.9}} \\
\end{array}$
\end{center}
 \caption{Pictures of subgraphs of the precision matrices induced by the 20 movies corresponding to the largest absolute partial correlations (averaged across difefrent $\lambda$ values).} \label{fig:subgraphs}
\end{figure}

Another very related application of the set-up described above is the one appearing in~\citet{AZM-11}, where 
one models the raw who-rated-whom binary data using a multivariate random effects model. 
\section{Conclusion and Remarks}
\label{sec:disc}
We propose a flexible, scalable and efficient algorithmic framework 
for large scale $\ell_1$ penalized inverse 
covariance selection problems that is used in several statistical applications. 
The framework gives rise to our main proposal \pine,
its close cousin \pgr and \pex --- all of them operate on the primal version of the 
problem (\ref{crit1}).
The key ingredient to scalability and efficiency requires a novel idea --- that of inexact-minimization 
over an \emph{exact}-one in the row/column blocks.  
The non-smoothness in the objective, positive definiteness of the precision matrices and the overlapping entries of the rows/columns 
necessitates a separate convergence analysis. We address this issue.
This observation immediately brings down the per-iteration complexity of the algorithm by an order of magnitude, from $O(p^3)$ to $O(p^2)$.
On problems of size $p=1-3$ K, our proposal 
performs extremely well when compared to state of the art methods designed for problem (\ref{crit1}). 
Our proposal tracks a sparse, positive definite precision matrix and its exact inverse i.e the covariance matrix at every iteration and is suited to return a solution with low/moderate accuracy depending upon the application task at hand. In particular, this makes it particularly suitable for large scale covariance selection problems where a very high accuracy solution is not practically feasible. 
\pine is particularly suitable for computing a path of solutions on a grid of $\lambda$ values using warm-starts --- and it performs better when compared to existing path algorithms. 

Our algorithm is supported by complexity analysis which shows that it is favorable over existing algorithms by an order of magnitude.
Our proposals are well supported by numerical experiments on real and synthetic data.

\paragraph{Exact Thresholding of Covariance Matrices}
Fairly recently \citet{MH-GL-11} proposed an exact
thresholding strategy which becomes useful if the non-zeros of the
solution $\B\Theta^*_{\lambda}$ to (\ref{crit1}) breaks down into
connected components. The \emph{same} components can be recovered
by looking at the non-zeros of the matrix $\B{\eta}_\lambda(\s)$,
where $\eta(\cdot)=\sgn(\cdot)(|\cdot| - \lambda)_+$ is the
component-wise soft-thresholding operator at $\lambda$.
As shown in \citet{MH-GL-11}, this strategy can be used as a 
wrapper around any algorithm for solving (\ref{crit1}) for
sufficiently large $\lambda$ so that it admits a decompostion into connected components.
Since the aforementioned strategy heavily relies on having a scalable algorithm for (\ref{crit1}), determined by the size of the maximal component --- our proposal opens the possibility of solving problems (\ref{crit1}) for an 
even wider range of $\lambda$-values. 

\paragraph{Extensions to other convex regularizers} Though we were concerned with (\ref{crit1}) in this paper, our framework can accommodate other variants of block-separable regularizers, in place of the $\ell_1$ norm on the entries of the matrix.
This includes:
\begin{enumerate}
\item The weighted $\ell_1$ norm  i.e. $ \sum_{ij} p_{ij}|\theta_{ij}|$, where $p_{ij} \geq 0, \forall i,j$ are given scalars. 
See \citet{FHT2007a,Fan09networkexploration} for use of this penalty.
\item The group lasso /node-sparse \citep{FHT-GL-10} norm on the blocks of the precision matrix:
$\sum_{i=1}^{p} \sqrt{\sum_{j \neq i}  \theta_{ij}^2}$.
\item The elastic net regularization i.e. 
$\alpha \sum_{ij} |\theta_{ij}| + (1- \alpha) \sum_{ij} \theta^2_{ij}$
\citep{ZH2005}
\end{enumerate}
These all are achieved by modifying Algorithm~\ref{algo:inner-cd}, with an inexact minimization strategy for the above penalties.


\commentout{
We did not use the screening rule for any of the algorithms used in
this paper. Moreover we compared different algorithms across a wide
range of $\lambda$ values from sparse to dense solutions (where the
screening effect is not likely to have much of an effect).  For the
real-life movie-lens problem with p=3207, we found that, on the grid
of $\lambda$ values chosen, the size of the maximal connected
component was always larger than 3100. The size soon became 3207 in
the interior of the grid of $\lambda$ values --- so it is not likely
to leverage extra mileage in the computational performance of the
algorithms.
}

\section{Acknowledgements}
We would like to thank Trevor Hastie (Stanford University) and Liang Zhang (Yahoo! Labs) for helpful discussions. Many thanks to Liang for setting up the Movie-Lens data-set for our experiments. Rahul Mazumder would like to thank Yahoo ! Research for hospitality during the summer of 2010, during which this work started.

\newpage

\begin{appendix}
\section{  \uppercase{ {\bf Appendix }} : Proofs and Technical Details}\label{append-sec}
This Section accumulates the technical details and proof details that were outlined in the text.
\subsection{Theorem \ref{thm:conv} : Set up and Proof}\label{proof-conv}
Firstly we will like to point out some important points about the convergence result which also sheds important light on the 
properties of the solution (\ref{crit1}).
If $\lambda>0$, the sequence of objective values and the estimates are bounded below (see Lemma \ref{lem-covsel}). Then
by standard results in real-analysis, the sequence of objective values converge to $g_\infty$ (say). It is not clear however 
(see \citet{Tseng01}, and references therein for counter-examples) 
that the point of convergence i.e. $g_\infty$ corresponds to the minimum of the problem (\ref{crit1}). 
Fortunately however, we show in this section that $g_\infty$ actually is the optimum of the minimization problem (\ref{crit1}). 

Observe that the convex optimization problem (\ref{crit1}), for $\lambda =0$ and $\s$ rank-deficient 
will have its infimum at $-\infty$. 
However, it turns out that for $\lambda >0$, this condition is avoided and the optimal value of the problem is finite. 

As is the case for many convex optimization problems \citep[see for example]{BV2004}, its is not necessary that 
problem  (\ref{crit1}) will have a unique minimizer. It turns out, however, that as soon as $\lambda>0$, problem (\ref{crit1})
has a unique minimizer. 
The assertions made above are consequences of Lemma \ref{lem-covsel}.

We need to set up a formal framework and present a few lemmas leading to the proof.

\paragraph{\textbf{The $\ell_1$ Regularized Proximal Map}} \label{sec:app:prox}
A variant of Step \ref{item-12}, in Algorithm \ref{algo:inner-cd} is one where we use a proximal step\citep{nest-07}, instead of one sweep of cyclical block-coordinate descent. Recall that the function $g_p(\B{\theta}_{1p})$ (\ref{margin-diag-2}) 
is in the composite form \citep{nest-07} :
\begin{equation}\label{compo}
g_p(\B{\theta}_{1p}) = f_p(\widetilde{\B{\theta}}_{1p}) + 2\lambda\|\B{\theta}_{1p}\|_1
\end{equation}
where $f_p(\cdot)$ denotes the smooth part given by:
$$f_p(\widetilde{\B{\theta}}_{1p})=\B{\theta}_{1p}'\{(s_{pp}+ \lambda)\B{\Theta}_{11}^{-1}\}\B{\theta}_{1p} + 2 \M{s}_{1p}'\B{\theta}_{1p}$$
It is easy to see that the gradient $\nabla f_p(\cdot)$ of the function $f_p(\cdot)$ is Lipschitz continuous i.e. :
\begin{equation}\label{def:lip-lip}
\|\nabla f_p(\M{u}) - \nabla f_p(\M{v})\|_2 \leq L_p \|\M{u} - \M{v}\|_2
\end{equation} 
and an estimate of $L_p=2(s_{pp} + \lambda)\|\B{\Theta}_{11}^{-1}\|_2$.
The proximal step or the generalized gradient step (in place of the coordinate-wise updates (\ref{inner-cd-update1}) ) is given by the following:
\begin{eqnarray}
\widehat{\B\omega} &=& \argmin_{\omega \in \Re^{p-1}}\; \{\frac{L_p}{2} \| \omega - (\widetilde{\B{\theta}}_{1p} - 
\frac{1}{L_p} \nabla f_p(\widetilde{\B{\theta}}_{1p}) \|_2^2 + 2\lambda\|\omega\|_1 \} - \widetilde{\B{\theta}}_{1p} \\
&=&\eta(\widetilde{\B{\theta}}_{1p} - \frac{1}{L_p} \nabla f_p(\widetilde{\B{\theta}}_{1p}) ; \frac{2\lambda}{L_p} ) - \widetilde{\B{\theta}}_{1p}
\label{block-proximal}
\end{eqnarray}
where 
$\nabla f_p(\widetilde{\B{\theta}}_{1p})=2(s_{pp}+ \lambda)(\widetilde{\B{\Theta}}_{11})^{-1}\widetilde{\B{\theta}}_{1p} + 2 \M{s}_{1p}$
and $\eta(\cdot; \gamma) = \sgn(\cdot)(|\cdot| -\gamma)_+$ is the soft-thresholding operator applied component-wise to a vector
$\cdot \in \Re^{p-1}$.

In what follows, we will study a minor variation in  
Step 2 of Algorithm~\ref{algo:inner-cd}. Instead of using one sweep of cyclical coordinate descent, we will use one proximal step as described in (\ref{block-proximal}). The convergence result with the cyclical coordinate-descent version is no different but simply adds to the technicality of the analysis. 

\paragraph{  Properties of the soft-thresholding operator}
Before going into the proof we need a lemma pertaining to an important property of the soft-thresholding operator i.e. 
the $\ell_1$ Regularized Proximal Map.

For a function $h:\Re^q \mapsto \Re$ with  Lipschitz continuous gradient:
\begin{equation}\label{Lipsh-cont}
\|\nabla h(\mathbf{x}) - \nabla h(\mathbf{y}) \| \leq L\| \mathbf{x}- \mathbf{y}\|  
\end{equation}
the following majorization property holds \citep[See for example, Lemma 2.1]{fista-09} 
\begin{equation}\label{major-1}
\frac{L}{2}\|\mathbf{w} -\mathbf{x}\|_2^2 +  \langle \nabla h(\mathbf{x}), \mathbf{w}- \mathbf{x}\rangle 
+ h(\mathbf{x}) \geq h(\mathbf{w})
\end{equation}
The minimizer wrt $\M{w}$ for the $\ell_1$ regularized problem:
\begin{equation}\label{major-def1}
 \mathrm{Maj}(\M{w};\M{x}) :=  \frac{L}{2}\|\mathbf{w}-\mathbf{x}\|_2^2 + \langle \nabla h(\mathbf{x}), \mathbf{w}- \mathbf{x}\rangle + h(\mathbf{x}) + \lambda \|\mathbf{w}\|_1 
\end{equation}
is given by the proximal map or the soft-thresholding operator:
\begin{align}\label{gen-thresh}
M(\mathbf{x}):= &\argmin_{\mathbf{w}} \;\; \{\;\; \frac{L}{2}\|\mathbf{w}-\mathbf{x}\|_2^2 +  
\langle \nabla h(\mathbf{x}), \mathbf{w}- \mathbf{x}\rangle + \lambda \|\mathbf{w}\|_1 \} \\
  = &\argmin_{\mathbf{w}} \;\;  \{ \;\; \frac{L}{2}\|\M{w}- (\mathbf{x}- \frac{1}{L}\nabla h(\mathbf{x}))\|_2^2 +  \lambda \|\mathbf{w}\|_1 \}  \nonumber\\
  =& \eta \left( (\mathbf{x}- \frac{1}{L}\nabla h(\mathbf{x})); \frac{\lambda}{L} \right )\nonumber
\end{align}
The following Lemma states an important property of the map $M(\mathbf{x})$. 
\begin{lemma}\label{lem:suff-decrease}
Consider the function  $H(\cdot)$ defined by:
\begin{equation}\label{compo-thm}
H(\mathbf{w}) = h(\M{w}) + \lambda \|\M{w}\|_1
\end{equation} 
with $h(\cdot)$ having the property in (\ref{Lipsh-cont}). 
For any $\mathbf{x} \in \Re^q$ and $M(\cdot)$ as defined in (\ref{gen-thresh}) the following holds:
\begin{equation}\label{suff-decrease-1}
\frac{2}{L}\cdot \big ( H(\M{x}) - H(M(\M{x})) \big)  \geq \| \M{x} - M(\M{x}) \|_2^2
\end{equation}
\end{lemma}
\begin{proof}
It can be shown using elementary convex analysis and the properties of the map $\mathrm{Maj}(\cdot)$ (\ref{major-def1}) defined above, 
\citep[See Lemma 2.3]{fista-09}:
\begin{equation}\label{major-fista}
H(\M{x}) - H(M(\M{y})) \geq \frac{L}{2}\|M(\M{y}) - \M{y}\|_2^2 + L \langle \M{y} - \M{x}, M(\M{y}) - \M{y}\rangle
\end{equation}
Substituting $\M{y} = \M{x}$ above we get the desired result in (\ref{suff-decrease-1}).
\end{proof}

\textbf{ Proof of Theorem \ref{thm:conv}, part (a)}
The monotonicity follows by construction of the sequence of iterates $\B\Theta_k$. 

The iterate $\B{\Theta}_{k}$ is obtained by updating all the $p$ rows/columns of the matrix  $\B\Theta$, cyclically.
We now introduce some notation.
Let us denote the successive row/column updates by:
\begin{eqnarray} \label{seq-defn}
\text{Update row/column 1}& \rightarrow  \B{\Theta}_{k,1} \\
\text{Update row/column 2}&  \rightarrow \B{\Theta}_{k,2}  \nonumber \\
\ldots\ldots\ldots  &\ldots \nonumber \\
\text{Update row/column p}& \rightarrow \B{\Theta}_{k,p} \nonumber
\end{eqnarray}
Further we use $\B{\Theta}_{k,i}[-i,i] \in \Re^{p-1}$ to denote the $i^{\mathrm{th}}$
column of the matrix $\B{\Theta}_{k,i}$ (excluding the diagonal entry).
We need to estimate the difference in $\B{\Theta}_{k,i-1}$ and $\B{\Theta}_{k,i}$ --- note that they differ only in the 
$i^{\mathrm{th}}$ row/column.

To settle ideas and using the same set-up as in Section \ref{our-method}, we concentrate on row/column $p$. 
The difference   
$\B{\Theta}_{k,p}[-p,p] - \B{\Theta}_{k,p-1}[-p,p]$ can be estimated by using Lemma~\ref{lem:suff-decrease}.
To see this recall the framework of Algorithm~\ref{algo:block-gl} as described in Section~\ref{our-method}.
To update the $p^{\mathrm{th}}$ row/column we need to consider a proximal-gradient step in the function $g_p(\cdot)$
as described in~(\ref{compo}). This function exactly fits into the framework of  
Lemma~\ref{lem:suff-decrease}, for specific choices of $L$, $h(\cdot)$, $H(\cdot)$ and $\lambda$.
Let the Lipshcitz constant at this iterate be denoted by $L_{k,p}$.  
Using the equality 
$g(\B{\Theta}_{k,p-1}) - g(\B{\Theta}_{k,p}) = g_p(\B{\Theta}_{k,p-1}[-p,p]) - g_p(\B{\Theta}_{k,p}[-p,p])$ (which follows by construction) 
and Lemma~\ref{lem:suff-decrease} we have:
\begin{equation}\label{estimate-diff-off}
\frac{2}{L_{k,p}} (g(\B{\Theta}_{k,p-1}) - g(\B{\Theta}_{k,p})  ) \geq \|\B{\Theta}_{k,p}[-p,p] - \B{\Theta}_{k,p-1}[-p,p]\|_2^2
\end{equation}
Recall that we established in Section~\ref{sec:prop:pd}, that 
Agorithm~\ref{algo:block-gl} produces a sequence of estimates $\B{\Theta}_{k,i}$ such that 
$\B{\Theta}_{k,i} \succ 0$. Further note that the minimum 
of~(\ref{crit1}) is finite
(as $\lambda >0$, Lemma \ref{lem-covsel}). It follows that 
there exists $\rho' > \rho>0$ such that  
\begin{equation}\label{eq:unif-bddness}
\rho' I_{p \times p} \succ \B{\Theta}_{k,i} \succ \rho I_{p \times p}, \forall k
\end{equation}
where $I_{p \times p}$ is a $p$ dimensional identity matrix. 
Since $L_{k,i}$ is a scalar multiple~(\ref{def:lip-lip}) of the spectral norm $\|(\B\Theta_{k,i}[-i,-i])^{-1}\|_2$, 
it follows from~(\ref{eq:unif-bddness}) that $\inf_{k,i} L_{k,i} >0$ and $\infty > \sup_{k,i} L_{k,i}$. 

Thus using the monotonicity of the sequence of objective values $g(\B\Theta_{k,i})$ 
for $i=1,\ldots, p$, $k \geq 1$, the 
fact that the minimum value of (\ref{crit1}) is finite and the boundedness of $\frac{1}{L_{k,i}}$, we see that 
the left hand side of (\ref{estimate-diff-off}) converges to zero as $k \rightarrow \infty$. 
This implies that $\B{\Theta}_{k,p}[-p,p] - \B{\Theta}_{k,p-1}[-p,p] \rightarrow 0$ as $k \rightarrow \infty$ i.e.
the successive difference of the off-diagonal entries converge to zero as $k \rightarrow \infty$. 
In particular, we have this to be true for every row/ column $i \in \{1,\ldots, p\}$ i.e. 
$$ \B{\Theta}_{k,i}[-i,i] - \B{\Theta}_{k,i-1}[-(i-1),i-1] \rightarrow 0, \;\;\; k \rightarrow \infty\;\;  \text{for every $i \in \{1,\ldots, p\}$}$$
In the above, we use the convention $\B{\Theta}_{k,0}[-1,1] = \B{\Theta}_{k-1,p}[-p,p]$.

Since $\{\B{\Theta}_k\}_{k}$ is a bounded sequence it has a limit point  --- let  $\B{\Theta}_\infty$ be a limit point.
Moving along a sub-sequence (if necessary), $n_k \subset \{1,2, \ldots\}$ we have $\B{\Theta}_{n_k} \rightarrow \B{\Theta}_\infty$.

Using the stationary condition for the update described in (\ref{block-proximal}), the $p^{\mathrm{th}}$ row/column (off-diagonal entries) satisfies: 
$$(s_{pp} + \lambda) (\B\Theta_\infty)^{-1}_{11} (\B{\Theta}_{\infty,p}[-p,p]) + \M{s}_{1p} 
+ \lambda \sgn( \B{\Theta}_{\infty,p}[-p,p] ) = 0. $$
The above holds true for every row/column $i \in \{1,\ldots, p\}$. 
Using the above stationary condition along with the update relation for the diagonal entries 
as in Step \ref{item-13} in Algorithm \ref{algo:inner-cd}, it is easy to see that the limit point $\B{\Theta}_\infty$ satisfies the global stationary condition for problem~(\ref{crit1}) i.e. 
$$ - (\B\Theta_\infty)^{-1} + \s + \lambda \sgn(\B\Theta_\infty) = \M{0}$$

Thus we have established that $g(\B\Theta_k)$ converges to the global minimum of the function $g(\cdot)$.

\textbf{Proof of Theorem \ref{thm:conv}, part (b)}
For this part it suffices to show that there is a unqiue limit point for the sequence $\B\Theta_{k}$. 
Note that we showed in Part (a) that every limit point of the sequence $\B\Theta_{k}$ is a minimizer for the problem (\ref{crit1}).
Now by Lemma \ref{lem-covsel}, there is a unique minimizer of (\ref{crit1}). This implies that  $\B\Theta_{k}$ has a unique 
limit point and hence the sequence converges to  $\B\Theta_{\infty}$, the minimizer of $g(\cdot)$.

\section{Complexity analysis details of \pine,\pex and  \pgr}\label{comp-complex-append}
\subsection{Complexity of \pine}\label{sec:pine}
Step 2 of Algorithm~\ref{algo:block-gl} requires computing $\B{\Theta}_{11}$, this requires 
$O((p-1)^2)$ (see Section \ref{sec:props}). Algorithm~\ref{algo:inner-cd} does one sweep of cyclical coordinate descent --- this has worst case complexity $O((p-1)^2)$ --- in case the solution to the $\ell_1$ regularized QP is sparse, the cost is much smaller. 
It should be noted here that any constant number of cycles (say $T_o$) of cyclical coordinate descent will lead to a complexity of
$O(T_o\cdot(p-1)^2)$. As long as $T_o \ll p$ (say $T_o =1/2$) this leads to $O(p^2)$. 
This is followed by updating the covariance matrix with $O((p-1)^2)$, via rank-one updates. Hence Step 2, for each row / column has a complexity of $O(p^2)$, for $p$ rows/columns this is $O(p^3)$. If $\kappa $ denotes the total number of full sweeps across all the rows and columns this leads to $O(\kappa p^3)$. In practice based on our experiments 
we found $\kappa = 2-10$ is sufficient for convergence till a fairly high tolerance.
While computing a path of solutions with warm-starts $\kappa$ is around 2-4 for different values of $\lambda$. The value of $\kappa$ increases when $\lambda$ is very small so that the resultant solution $\B\Theta^*$ is dense --- but since these $\lambda$ values correspond to almost un-regularized likelihood solutions, in most applications they are not statistically interesting solutions.

\subsection{Complexity of \pex}\label{sec:pex}
In case of using \pex the analysis is quite similar to above but there are subtle differences. 
The complexity of matrix rank-updates remain the same $O(p^2)$, what changes is the number of coordinate sweeps required for
Algorithm \ref{algo:inner-cd} to solve the inner $\ell_1$ regularized block QP till high accuracy. This problem is fairly challenging in its own right and is computationally hard when $p$ is a few thousand. 
Precise convergence rates of coordinate descent to the best of our knowledge are not known. This depends largely upon the data type being used. Often the number of coordinate sweeps i.e. $k$ can be $O(p)$ --- leading to a complexity of $O(p^3)$.
If (generalized) gradient descent methods \citep{nest-07} are used instead of cyclical coordinate descent --- then the number of iterations
$k$ is of the order of $O(1/\epsilon)$, where $\epsilon >0$ is the accuracy of the solution. For $\epsilon \approx 1/p$, 
$k \approx p$. 
Thus \pex has roughly a complexity of $O(p^3)$ for every row/column update --- leading to an overall cost of $O(p^4)$ for one full sweep across all rows/columns. If there are $\kappa'$ full sweeps across all rows/columns the total cost is $O(\kappa' p^4)$. 

In case $\lambda$ is  large enough so that every $\ell_1$ regularized QP can be solved quite fast i.e. $O(p^2)$ --- the total cost of
\pex reduces to $O(p^3)$.

\subsection{Complexity of \pgr}  \label{sec:pgr}
The main difference of \pgr lies in the manner in which it updates the rows/columns via appending rows/columns in Steps 1-3 of Algorithm \ref{algo:grow}. Steps 1-3 have a cost of $\sum^{p}_{m=1} m^2$, which is approximately $p^3/3$. When compared with one sweep of 
\pine, the growing step is approximately one-third of the cost of \pine. One sweep of the growing strategy leads to inferior performance when compared to one sweep of \pine. However, after a smaller number of sweeps, \pgr can obtain better likelihoods than \pine.
In some examples, as seen in the experimental 
section, \pgr is faster than \pine in obtaining a solution with low accuracy. 


\newpage

\section{\uppercase{ \textbf{ Supplementary Materials } } }\label{sec:suppl}
This portion gathers some of the technicalities avoided in the main text of the article and the Appendix \ref{append-sec}.  
\subsection{Properties of the updates of Algorithm~\ref{algo:block-gl}} \label{app:sec:prop}
We present here a detailed derivation of the properties of the updates of Algorithm~\ref{algo:block-gl}.
\subsubsection{Positive-definiteness}\label{sec:pos-def}
The updates described in Steps (\ref{item-11}), (\ref{item-12}), (\ref{item-13}) in Algorithm \ref{algo:inner-cd}
actually preserve positive definiteness of 
$\widehat{\B\Theta}$, under the assumption that $\widetilde{\B{\Theta}} \succ \mathbf{0}$.
Using the decomposition for $\widetilde{\B\Theta}$ in (\ref{break-x}), it follows from standard properties of positive definiteness of block partitioned matrices \citep{BV2004} that:
\begin{equation}\label{pos-def-1}
\widetilde{\B{\Theta}} \succ 0  \;\;\;\;\;  \text{iff} \;\;\;\;\; \widetilde{\B{\Theta}}_{11} \succ 0 ,\;\; 
\widetilde{\theta}_{pp} - \widetilde{\B{\theta}}_{1p}'\widetilde{\B{\Theta}}^{-1}_{11}\widetilde{\B{\theta}}_{1p} > 0
\end{equation}
Observe that $\widehat{\B{\Theta}}_{11} =\widetilde{\B{\Theta}}_{11}$, by construction.  
Further by the property of the update Step \ref{item-13} (Algorithm~\ref{algo:block-gl})   we see that 
$$\widehat{\theta}_{pp} - \widehat{\B{\theta}}_{1p}'\widetilde{\B{\Theta}}^{-1}_{11}\widehat{\B{\theta}}_{1p} = \frac{1}{s_{pp}+ \lambda} >0.$$
 This shows
by (\ref{pos-def-1}) that $\widehat{\B\Theta} \succ 0$.

A simple consequence of the above observation is that  
$\log\det(\widehat{\B{\Theta}})$ is finite if $\log\det(\widetilde{\B{\Theta}})$ is so.

\subsubsection{Tracking $(\hat{\B{\Theta}}, (\hat{\B{\Theta}})^{-1})$}\label{sec:track}
For updating the $p^{\mathrm{th}}$ row/column, \pine  requires knowledge of 
$(\widetilde{\B\Theta}_{11})^{-1}$. 
Of course, it is not desirable to compute the inverse from scratch
for every row/column $i$, with a complexity of $O(p^3)$. 
Assume that, before operating on the $p^{\mathrm{th}}$ row/column we already have with us the tuple
$(\widetilde{\B\Theta},(\widetilde{\B\Theta})^{-1})$ --- then it is fairly easy to compute 
$(\widetilde{\B\Theta}_{11})^{-1}$ via:
\begin{equation}\label{update-rank-one-1}
(\widetilde{\B\Theta}_{11})^{-1} = \widetilde{\M{W}}_{11} - \widetilde{\M{w}}_{1p}\widetilde{\M{w}}_{p1}/\widetilde{w}_{pp}, 
\end{equation}
where $\widetilde{\M{W}}:=(\widetilde{\B\Theta})^{-1}$ and the blocks
$\widetilde{\M{W}}_{11},\widetilde{\M{w}}_{1p},\widetilde{\M{w}}_{p1}, \widetilde{w}_{pp}$ of the matrix 
$\widetilde{\M{W}}$, have the same structure as in (\ref{break-x}). This follows by 
standard-formulae of inverses of block-partitioned matrices --- and the update requires $O(p^2)$.

Once the $p^{\mathrm{th}}$ row/column of the matrix $\widetilde{\B\Theta}$ is updated, we obtain $\widehat{\B{\Theta}}$. The matrix 
$\widehat{\M{W}}:=(\widehat{\B{\Theta}})^{-1}$ is obtained via:
\begin{equation}\label{update-rank-one-2}
\widehat{\M{W}}_{11} = (\widetilde{\B{\Theta}}_{11})^{-1} - 
\frac{(\widetilde{\B{\Theta}}_{11})^{-1}\widehat{\B{\theta}}_{1p}\widehat{\B{\theta}}_{p1}(\widetilde{\B{\Theta}}_{11})^{-1}}
{(\widehat{\theta}_{pp} - \widehat{\B{\theta}}_{p1} (\widetilde{\B{\Theta}}_{11})^{-1}\widehat{\B{\theta}}_{1p})} ; \;\;  
\widehat{\M{w}}_{1p} = -\frac{(\widetilde{\B{\Theta}}_{11})^{-1}\widehat{\B{\theta}}_{1p}}{(\widehat{\theta}_{pp} - \widehat{\B{\theta}}_{p1} (\widetilde{\B{\Theta}}_{11})^{-1}\widehat{\B{\theta}}_{1p})}, 
\end{equation}
where as before the blocks of the matrix $\widehat{\M{W}}$, follow the same notation as in (\ref{break-x}).
The cost is again $O(p^2)$.
Note that the diagonal entry $\widehat{w}_{pp} =  1/(\widehat{\theta}_{pp} - \widehat{\B{\theta}}_{p1} (\widetilde{\B{\Theta}}_{11})^{-1}\widehat{\B{\theta}}_{1p})$.

The above recursion shows how to track $(\widehat{\B\Theta}, \widehat{\B\Theta}^{-1})$
 (as well as $(\widetilde{\B\Theta}, \widetilde{\B\Theta}^{-1})$) as one cycles across the rows/columns of the matrix $\B\Theta$.

\subsection{Algorithmic Description of Primal Growth for Graphical Lasso \pgr}\label{sec:grow:algo} 
This elaborates Section~\ref{sec:grow}, in the text. 
Given an initial working dimension $p_0$ (typically $p_0=1$) and estimates of the precision and the covariance matrix
$(\widetilde{\B{\Theta}}_{p_0 \times p_0},(\widetilde{\B{\Theta}}_{p_0 \times p_0})^{-1})$, Algorithm~\ref{algo:grow} describes the task of obtaining the solution to (\ref{crit1}), with $\B\Theta, \s$ having dimensions $p \times p$. 
\begin{algorithm}[tb]
\caption{\pine with Growing Strategy : \pgr} 
\label{algo:grow} 
Inputs: $\lambda$, $\s_{p \times p}$ and  
 $p_0 \times p_0$ matrices $(\widetilde{\B{\Theta}},(\widetilde{\B{\Theta}})^{-1})$, where 
$p_0<p$.

Set initial working row/column $m=p_0+1$.
\begin{enumerate} 
\item[1] 
Consider a  
$m \times m$ dimensional problem of the form (\ref{crit1}) with covariance\footnote{Here $\s_{1:m \times 1:m}$, denotes the 
sub-matrix of $\s$ with row/column indices $1,\ldots,m$.} $\s_{1:m \times 1:m}$
and initializations     
$(\widetilde{\B{\Theta}},(\widetilde{\B{\Theta}})^{-1})$, of dimension $m \times m$ where 
\begin{equation}\label{pad-1}
\widetilde{\B{\Theta}} \leftarrow  \mathrm{blkdiag}(\widetilde{\B{\Theta}},\frac{1}{(s_{m m} + \lambda)}); \;\;\;\;\;
(\widetilde{\B{\Theta}})^{-1} \leftarrow \mathrm{blkdiag}((\widetilde{\B{\Theta}})^{-1}, s_{mm} + \lambda)
\end{equation}
\item[2] Apply Step 2 of Algorithm \ref{algo:block-gl}
with inputs $\s_{1:m \times 1:m}$, $(\widetilde{\B{\Theta}}, (\widetilde{\B{\Theta}})^{-1})$ and input dimension $m$.

The above results in $(\widehat{\B{\Theta}}, (\widehat{\B{\Theta}})^{-1})$ --- both $m$ dimensional matrices.

Assign $(\widetilde{\B{\Theta}}, (\widetilde{\B{\Theta}})^{-1}) \leftarrow (\widehat{\B{\Theta}}, (\widehat{\B{\Theta}})^{-1})$
\item[3] Increase $m=m+1$, and go to Step-1 (if $m \leq p$); else go to Step-4. 
\item[4] Apply Algorithm~\ref{algo:block-gl} with   
$\s_{1:p \times 1:p}$, and initializations $(\widetilde{\B{\Theta}}_{p \times p}, (\widetilde{\B{\Theta}}_{p \times p})^{-1})$, 
till a desired tolerance. The output is the solution to problem (\ref{crit1}).
\end{enumerate} 
\end{algorithm} 

\subsection{Performances of \pine and its variants for low--high accuracy solutions}\label{sec:low-high-accuracy}
We now proceed to show how gracefully they deliver solutions of lower accuracy within a much shorter span of time
making it feasible to scale to very high dimensional problems.
As mentioned earlier, the primal formulation is particularly suited for this task of delivering solutions with lower convergence tolerance,
since it delivers a sparse and positive definite precision matrix and its exact inverse.  
Table \ref{tab:three} shows average timings in seconds across a grid of ten $\lambda$ values with varying degrees of accuracy.
In case the application demands a relatively low accuracy solution, then the algorithms deliver solutions within fractions of the time taken to deliver a solution with higher accuracy.

\begin{table}[htpb!]
  \begin{tabular}{| l | c |  c  | c c c | }     \hline
\multirow{2}{*}{p / N }    & average \% &  Accuracy & \multicolumn{3}{|c|}{Algorithm Times (sec)}   \\
   &                               of  nnz &   (TOL)              & \pgr & \pex & \pine   \\ \hline 
\multirow{3}{*}{$1\times 10^3$ / $2\times 10^3$} &\multirow{3}{*}{61.3} &$10^{-2}$ & 60.16 & 68.29 & \bf 49.0 \\
 && $10^{-3}$ &  80.59 &  104.67 &  \bf 67.31 \\
&& $10^{-4}$  & 112.31 &  134.13 &  \bf 91.91 \\
&&   $10^{-5}$& 140.08 &  174.52 & \bf 120.75 \\ \hline
\multirow{3}{*}{$1\times 10^3$ / $1\times 10^3$} & \multirow{3}{*}{62.77} & $10^{-2}$  & 68.46&   92.32 & \bf 64.56 \\
 && $10^{-3}$   &\bf 95.49 &  126.25 &  \bf 95.47 \\
&& $10^{-4}$&  128.94 &  167.35& \bf 127.11 \\
&& $10^{-5}$  & \bf 174.44 &  199.36 &  177.06\\ \hline
\multirow{3}{*}{$1\times 10^3$ / $.5\times 10^3$} & \multirow{3}{*}{66.67}& $10^{-2}$ &  82.94 & 117.97 & \bf 65.94 \\
 &&$10^{-3}$ & 127.17 &  153.29 &  \bf 96.69\\
 && $10^{-4}$& 181.99 &  195.19& \bf 140.37 \\
 &&$10^{-5}$  & 252.95 &  234.97 & \bf 200.43\\ \hline
\multirow{3}{*}{$1.5\times 10^3$ / $4\times 10^3$} & \multirow{3}{*}{62.44} & $10^{-2}$ & 149.89&  195.19& \bf 124.55  \\
&&  $10^{-3}$ & \bf 189.52&  317.64&  204.77 \\
 && $10^{-4}$ &\bf 266.63 &  396.82 &  275.75 \\
 && $10^{-5}$ &  344.26 &  460.29 & \bf 333.55 \\ \hline
\multirow{3}{*}{$1.5\times 10^3$ / $3\times 10^3$}   & \multirow{3}{*}{62.78} &$10^{-2}$ &  145.59 &  215.65 & \bf 141.87 \\  
  && $10^{-3}$ & 203.86 &  300.57 & \bf 201.62 \\ 
 && $10^{-4}$ &\bf 261.12 & 397.72 &  271.34 \\ 
 &&$10^{-5}$ & \bf 344.19 &  477.64 &  346.47 \\ \hline
\multirow{3}{*}{$1.5\times 10^3$ / $2\times 10^3$}& \multirow{3}{*}{63.11}&$10^{-2}$ &  \bf 149.13 &  251.51 &  169.37 \\ 
                                                   &&       $10^{-3}$ & 250.02 & 354.75&  \bf 238.57 \\
&&  $10^{-4}$ &319.36 & 445.38&\bf 317.51\\
&& $10^{-5}$ & 414.91 &  523.38 & \bf 408.77 \\ \hline
  \end{tabular} 
\caption{Table showing average timings in seconds across a grid of ten $\lambda$ values with varying degrees of accuracy i.e. TOL. 
The column under average \% of non-zeroes denotes the \% of non-zeros in the optimal solution, averaged across the ten $\lambda$ values. No warm-start across $\lambda$'s is used. }  \label{tab:three}
\end{table}

Table \ref{tab:three} shows that \pine, \pex and \pgr return lower accuracy solutions to 
(\ref{crit1}) --- in times which are much less than that taken to obtain higher accuracy solutions. Note that even the lower accuracy solutions correspond to sparse and positive definite precision matrices, with guarantees on `TOL'. 
Even more interesting is the flexibility of methods like \pine, \pgr to obtain sparse and positive definite
solutions at low computational cost when compared to the times taken by the dual algorithms in Table \ref{tab:one}.
These favorable comparative timings go on to support 
our claim of making large scale covariance selection very practical. \pex turns out to the slowest among the three, \pgr and \pine are quite strong competitors, with \pine winning in majority of the situations. 

\subsection{Graphical display of sparsity patterns in the Movie-lens Graphs}\label{sec:spy}
This is an elaborate version of Section~\ref{sec:realmlens} in the main text.
Figure~\ref{fig-spy} represents the sparsity structures of the precision graphs obtained from the movie-movie 
similarities. The graph structures are displayed under the Sparse reverse Cuthill-McKee 
ordering\citep{spy-perm}\footnote{For a sparse symmetric matrix $A$ the reverse Cuthill-McKee ordering is 
a permutation $\pi$ such that $A(\pi,\pi)$ tends to have its nonzero elements closer to the diagonal. This is often used for visualizing the sparsity structure of large dimensional graphs} of the precision matrices 
as delivered by our algorithm \pine for three different values of $\lambda$. The presence of a `dot' in the figure represents a non-zero edge weight in the corresponding movie-movie precision graph. 
The percentages of (off-diagonal) non-zeros are presented below each figure. 
The tapering nature of the graphs for larger values of $\lambda$, show that the movies towards the extreme ends of the axes 
tend to be connected to few other movies. These movies tend to be conditionally dependent on very few other movies. 
The higher density of the points towards the center (of the left two figures) show that those movies tend to be connected to a larger number of other movies. 
\begin{figure}[ht]
  \centering
 \begin{psfrags}
 \psfrag{rhoid9}[][b]{\small{percentage zeros: 98.9 }}
\psfrag{ConnCompSizes}[][cb]{\small{$\log_{10}$(Size of Components)}}
\includegraphics[width=.31\textwidth,height=2.7in]{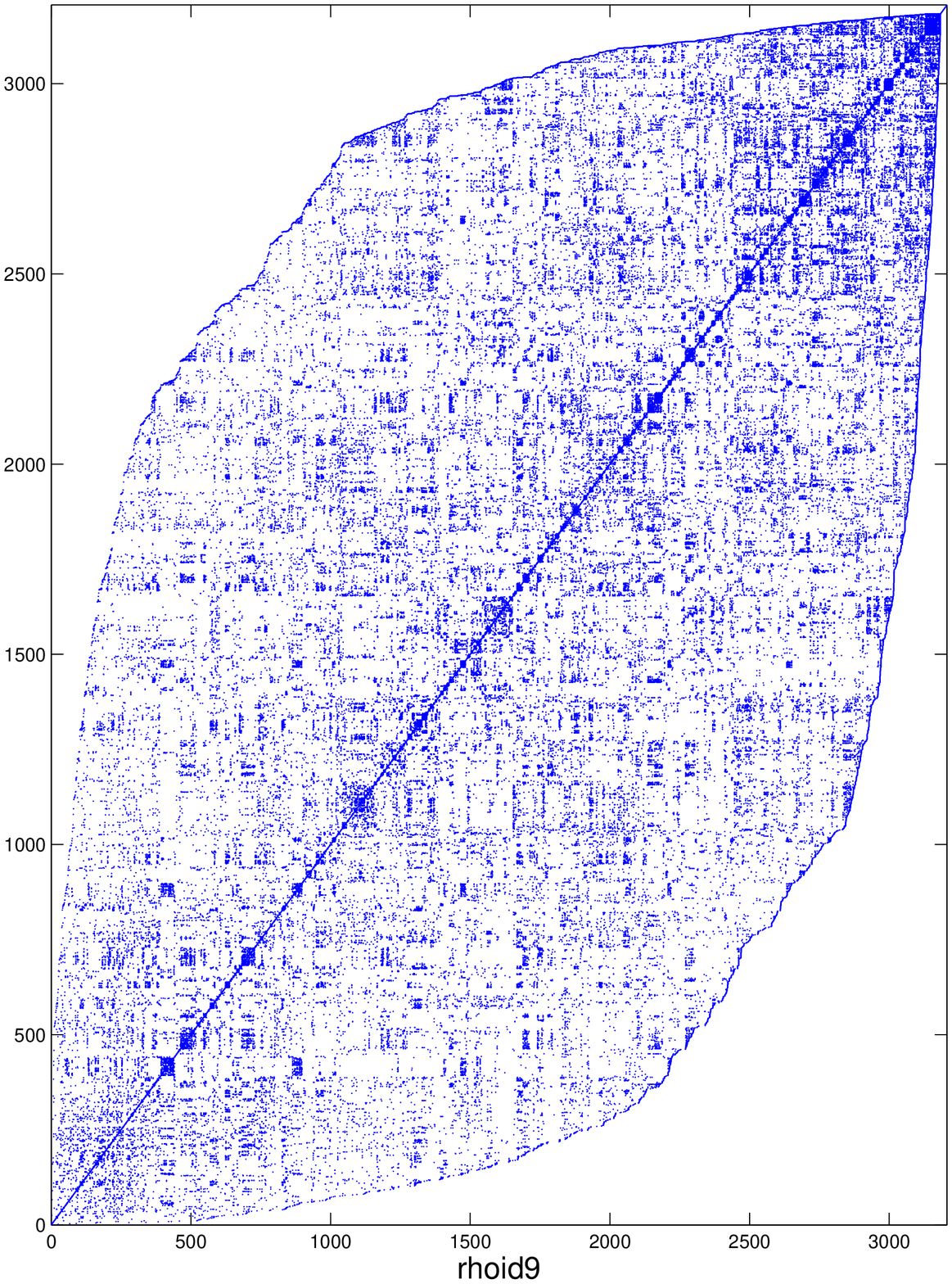}
\end{psfrags}
\begin{psfrags}
 \psfrag{BROWNDATA}[][ct]{\textsc{(B)}  \hspace{.01in} \scriptsize{p=4718} }
 \psfrag{RHO}{}
\psfrag{ConnCompSizes}[][cb]{\small{$\log_{10}$(Size of Components)}}
 \psfrag{rhoid7}[][b]{\small{percentage zeros: 98.7 }}
 \includegraphics[width=.31\textwidth,height=2.7in]{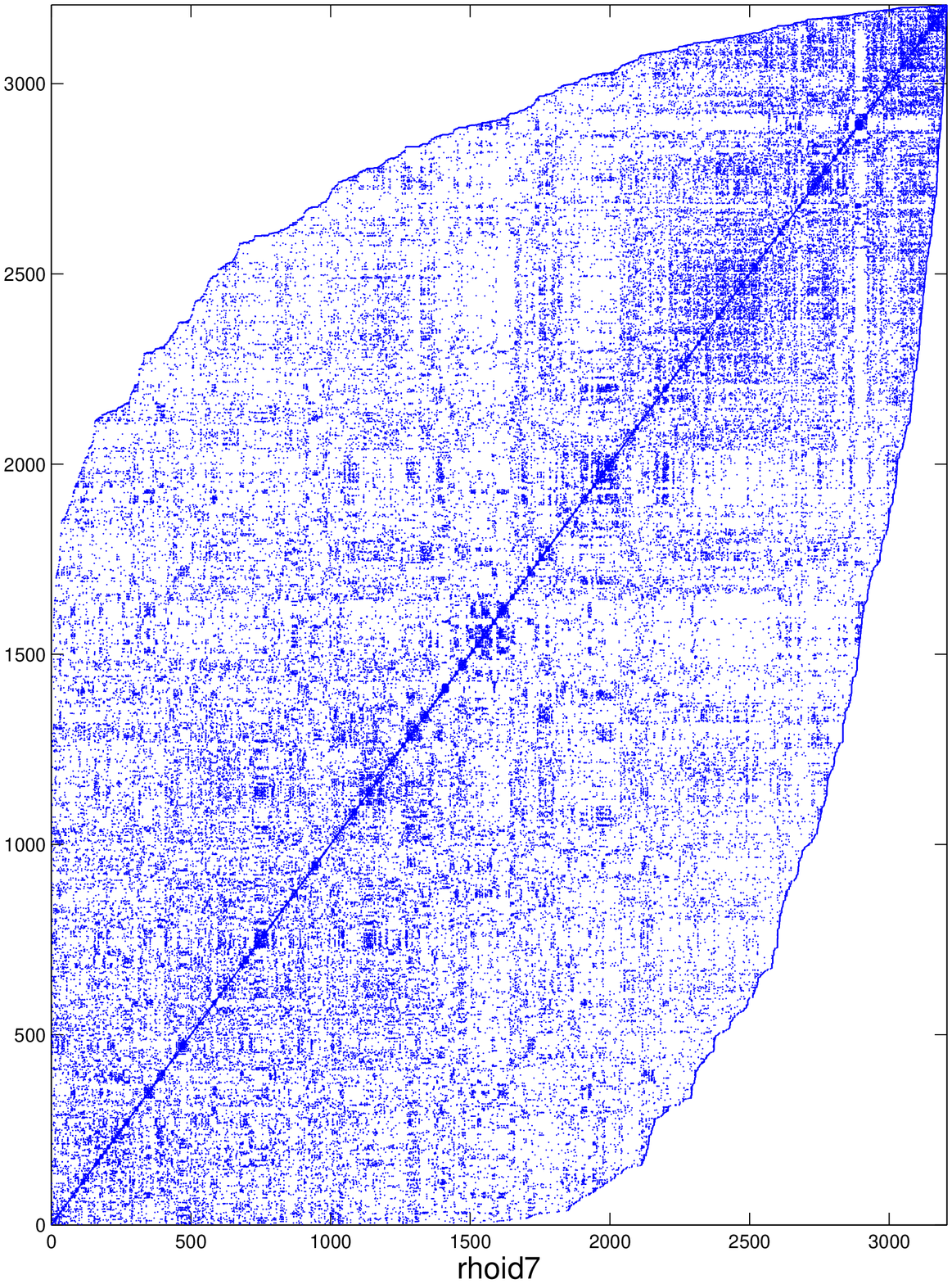}
\end{psfrags}
\begin{psfrags}
 \psfrag{NKIDATA}[][ct]{\textsc{(C)}  \hspace{.01in} \scriptsize{p=24281} }
 \psfrag{rhoid6}[][b]{\small{percentage zeros: 97.8 }}
\psfrag{ConnCompSizes}[][cb]{\small{$\log_{10}$(Size of Components)}}
 \includegraphics[width=.31\textwidth,height=2.7in]{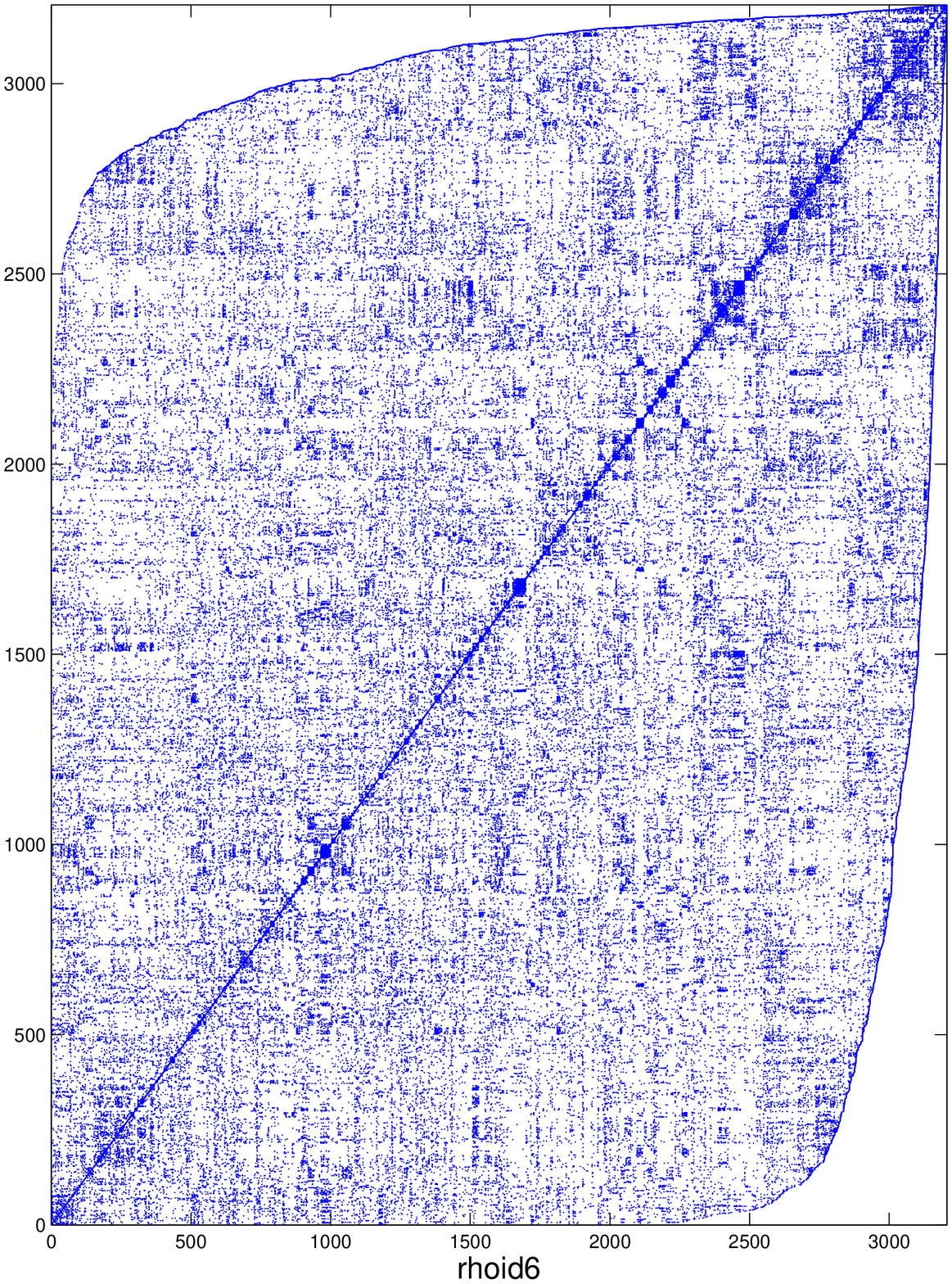}
 \end{psfrags}
  \caption{MATLAB {\tt spy} plots under the reverse Sparse reverse Cuthill-McKee ordering of the vertices
 of the sparse precision matrices obtained via \pine, for three different values of the tuning parameters. A dot represents presence of an edge. The percentage of off-diagonal zeros in the matrix are given below each plot.}
\label{fig-spy}
\end{figure}

\subsection{Movie-ID to Movie mapping Table}\label{sec:movie-map}
The movie ids--- movie name mapping is given in the following table:

\begin{table}[htpb!]
\scalebox{0.8}{
\begin{tabular}{| l | l | }     \hline
(0) {\tt PuppetMaster5: TheFinalChapter (1994) }  & (1) {\tt PuppetMaster4 (1993) }  \\
(2) {\tt Carnosaur3: PrimalSpecies (1996) } & (3) {\tt Carnosaur2(1995) }   \\ 
 (4) { \tt Fridaythe13thPartV: ANewBeginning (1985) } & (5) { \tt Fridaythe13thPartVII:TheNewBlood(1988) } \\ 
 (6) {\tt Porky'sRevenge (1985) } & (7) { \tt Porky's II: TheNextDay (1983) }\\ 
(8) { \tt SororityHouseMassacre (1986) } & (9) { \tt SororityHouseMassacreII (1990) } \\ 
 (10)  {\tt PoliceAcademy5: Assignment:MiamiBeach(1988) } & (11) {\tt PoliceAcademy6:CityUnderSiege(1989) } \\ 
 (12){\tt RockyIV(1985)} & (13) {\tt  RockyIII (1982) } \\ 
(14)  {\tt  Hellbound:HellraiserII(1988) } & (15) {\tt  Hellraiser(1987) } \\
 (16) {\tt  CloseShave,A(1995) } & (17) {\tt  WrongTrousers,The (1993) } \\ 
 (18) {\tt  Godfather:PartII,The(1974) } & (19) {\tt  Godfather,The(1972) }  \\ \hline
\end{tabular}}  
\caption{ Table showing thenames of the top 20 Movies, appearing in the top twenty strongest partial correlations. It is seen from 
Figure~\ref{fig:subgraphs} that edges often occur between movies and their sequels.}\label{tab:movies}
\end{table}

\end{appendix}

\bibliographystyle{biometrika} 
\bibliography{/h0c/rahulm/Dropbox/paper_glasso/new_agst_new.bib}

\begin{thebibliography}{36}
\expandafter\ifx\csname natexlab\endcsname\relax\def\natexlab#1{#1}\fi

\bibitem[{Agarwal et~al.(2011)Agarwal, Zhang \& Mazumder}]{AZM-11}
\textsc{Agarwal, D.}, \textsc{Zhang, L.} \& \textsc{Mazumder, R.} (2011).
\newblock Modeling item-item similarities for personalized recommendations on
  yahoo! front page.
\newblock \textit{Annals of Applied Statistics} \textbf{5(3)}, 1839--1875.

\bibitem[{Banerjee et~al.(2008)Banerjee, Ghaoui \& d'Aspremont}]{BGA2008}
\textsc{Banerjee, O.}, \textsc{Ghaoui, L.~E.} \& \textsc{d'Aspremont, A.}
  (2008).
\newblock Model selection through sparse maximum likelihood estimation for
  multivariate gaussian or binary data.
\newblock \textit{Journal of Machine Learning Research} \textbf{9}, 485--516.

\bibitem[{Beck \& Teboulle(2009)}]{fista-09}
\textsc{Beck, A.} \& \textsc{Teboulle, M.} (2009).
\newblock A fast iterative shrinkage-thresholding algorithm for linear inverse
  problems.
\newblock \textit{SIAM J. Imaging Sciences} \textbf{2}, 183--202.

\bibitem[{Bernardinelli \& Montomoli(1992)}]{besag}
\textsc{Bernardinelli, L.} \& \textsc{Montomoli, C.} (1992).
\newblock Empirical bayes versus fully bayesian analysis of geographical
  variation in disease risk.
\newblock \textit{Statistics in Medicine} \textbf{11}, 983–1007.

\bibitem[{Bottou \& Bousquet.(2008)}]{boto-08}
\textsc{Bottou, L.} \& \textsc{Bousquet., O.} (2008).
\newblock The trade-offs of large scale learning.
\newblock In \textit{Advances in Neural Information Processing Systems 20},
  J.~Platt, D.~Koller, Y.~Singer \& S.~Roweis, eds. Cambridge, MA: MIT Press,
  pp. 161--168.

\bibitem[{Boyd et~al.(2011)Boyd, Parikh, Chu, Peleato \& Eckstein}]{boyd-admm}
\textsc{Boyd, S.}, \textsc{Parikh, N.}, \textsc{Chu, E.}, \textsc{Peleato, B.}
  \& \textsc{Eckstein, J.} (2011).
\newblock Distributed optimization and statistical learning via the alternating
  direction method of multipliers.
\newblock \textit{Foundations and Trends in Machine Learning} \textbf{3(1)},
  1--122.

\bibitem[{Boyd \& Vandenberghe(2004)}]{BV2004}
\textsc{Boyd, S.} \& \textsc{Vandenberghe, L.} (2004).
\newblock \textit{Convex Optimization}.
\newblock Cambridge University Press.

\bibitem[{Cai et~al.(2011)Cai, Liu \& Luo}]{sp_l1_cai-11}
\textsc{Cai, T.}, \textsc{Liu, W.} \& \textsc{Luo, X.} (2011).
\newblock A constrained l1 minimization approach to sparse precision matrix
  estimation.
\newblock \textit{Journal of the American Statistical Association}
  \textbf{106}, 594--607.

\bibitem[{Cox \& Wermuth(1996)}]{cox-W-96}
\textsc{Cox, D.} \& \textsc{Wermuth, N.} (1996).
\newblock \textit{Multivariate Dependencies}.
\newblock Chapman and Hall, London.

\bibitem[{Fan et~al.(2009)Fan, Feng \& Wu}]{Fan09networkexploration}
\textsc{Fan, J.}, \textsc{Feng, Y.} \& \textsc{Wu, Y.} (2009).
\newblock Network exploration via the adaptive lasso and scad penalties.
\newblock \textit{Annals of Applied Statistics} , 521--541.

\bibitem[{Friedman et~al.(2007{\natexlab{a}})Friedman, Hastie, Hoefling \&
  Tibshirani}]{FHT2007}
\textsc{Friedman, J.}, \textsc{Hastie, T.}, \textsc{Hoefling, H.} \&
  \textsc{Tibshirani, R.} (2007{\natexlab{a}}).
\newblock Pathwise coordinate optimization.
\newblock \textit{Annals of Applied Statistics} \textbf{2}, 302--332.

\bibitem[{Friedman et~al.(2007{\natexlab{b}})Friedman, Hastie \&
  Tibshirani}]{FHT2007a}
\textsc{Friedman, J.}, \textsc{Hastie, T.} \& \textsc{Tibshirani, R.}
  (2007{\natexlab{b}}).
\newblock Sparse inverse covariance estimation with the graphical lasso.
\newblock \textit{Biostatistics} \textbf{9}, 432--441.

\bibitem[{Friedman et~al.(2010)Friedman, Hastie \& Tibshirani}]{FHT-GL-10}
\textsc{Friedman, J.}, \textsc{Hastie, T.} \& \textsc{Tibshirani, R.} (2010).
\newblock Applications of the lasso and grouped lasso to the estimation of
  sparse graphical models.

\bibitem[{Gilbert et~al.(1992)Gilbert, Moler \& Schreibe}]{spy-perm}
\textsc{Gilbert, J.~R.}, \textsc{Moler, C.} \& \textsc{Schreibe, R.} (1992).
\newblock Sparse matrices in matlab: Design and implementation.
\newblock \textit{SIAM Journal on Matrix Analysis} .

\bibitem[{Hoff(2009)}]{hoff}
\textsc{Hoff, P.~D.} (2009).
\newblock Multiplicative latent factor models for description and prediction of
  social networks.
\newblock \textit{Comput. Math. Organ. Theory} \textbf{15}, 261--272.

\bibitem[{Hoff~PD(2002)}]{hoffrafteryhandcock}
\textsc{Hoff~PD, Raftery~AE, H.~M.} (2002).
\newblock Latent space approaches to social network analysis.
\newblock \textit{Journal of the American Statistical Association} \textbf{97},
  1090--1098.

\bibitem[{Huang et~al.(2010)Huang, Li, Sun, Ye, Fleisher, Wu, Chen \&
  Reiman.}]{neuro-alzh-10}
\textsc{Huang, S.}, \textsc{Li, J.}, \textsc{Sun, L.}, \textsc{Ye, J.},
  \textsc{Fleisher, A.}, \textsc{Wu, T.}, \textsc{Chen, K.} \& \textsc{Reiman.,
  E.} (2010).
\newblock Learning brain connectivity of alzheimers disease by sparse inverse
  covariance estimation.
\newblock \textit{NeuroImage} \textbf{50}, 935--949.

\bibitem[{Lam \& Fan(2009)}]{fan-09}
\textsc{Lam, C.} \& \textsc{Fan, J.} (2009).
\newblock Sparsistency and rates of convergence in large covariance matrix
  estimation.
\newblock \textit{Annals of Statistics} \textbf{37(6B)}, 4254--4278.

\bibitem[{Lauritzen(1996)}]{Laur1996}
\textsc{Lauritzen, S.} (1996).
\newblock \textit{Graphical Models}.
\newblock Oxford University Press.

\bibitem[{Lu(2009)}]{Lu:09}
\textsc{Lu, Z.} (2009).
\newblock Smooth optimization approach for sparse covariance selection.
\newblock \textit{SIAM J. on Optimization} \textbf{19}, 1807--1827.

\bibitem[{Lu(2010)}]{Lu:10}
\textsc{Lu, Z.} (2010).
\newblock Adaptive first-order methods for general sparse inverse covariance
  selection.
\newblock \textit{SIAM J. Matrix Anal. Appl.} \textbf{31}, 2000--2016.

\bibitem[{Lu et~al.(2009)Lu, Agarwal \& Dhillon}]{lurecsys09}
\textsc{Lu, Z.}, \textsc{Agarwal, D.} \& \textsc{Dhillon, I.~S.} (2009).
\newblock A spatio-temporal approach to collaborative filtering.
\newblock In \textit{RecSys}.

\bibitem[{Mazumder \& Hastie(2011)}]{MH-GL-11}
\textsc{Mazumder, R.} \& \textsc{Hastie, T.} (2011).
\newblock Exact covariance thresholding into connected components for
  large-scale graphical lasso.
\newblock \textit{$arXiv:1108.3829v2$, (submiited)} .

\bibitem[{Meinshausen \& B\"{u}hlmann(2006)}]{MB2006}
\textsc{Meinshausen, N.} \& \textsc{B\"{u}hlmann, P.} (2006).
\newblock High-dimensional graphs and variable selection with the lasso.
\newblock \textit{Annals of Statistics} \textbf{34}, 1436--1462.

\bibitem[{Nesterov(2003)}]{nest_03}
\textsc{Nesterov, Y.} (2003).
\newblock Introductory lectures on convex optimization: Basic course.
\newblock \textit{Kluwer, Boston} .

\bibitem[{Nesterov(2005)}]{nest_05}
\textsc{Nesterov, Y.} (2005).
\newblock Smooth minimization of non-smooth functions.
\newblock \textit{Math. Program., Serie A} \textbf{103}, 127--152.

\bibitem[{Nesterov(2007)}]{nest-07}
\textsc{Nesterov, Y.} (2007).
\newblock Gradient methods for minimizing composite objective function.
\newblock Tech. rep., Center for Operations Research and Econometrics (CORE),
  Catholic University of Louvain.
\newblock Tech. Rep, 76.

\bibitem[{Ravikumar et~al.(2011)Ravikumar, Raskutti, Wainwright \&
  Yu}]{ravi-11}
\textsc{Ravikumar, P.}, \textsc{Raskutti, G.}, \textsc{Wainwright, M.~J.} \&
  \textsc{Yu, B.} (2011).
\newblock High-dimensional covariance estimation by minimizing l1-penalized
  log-determinant.
\newblock \textit{Electronic Journal of Statistics(to appear)} .

\bibitem[{Rothman et~al.(2010)Rothman, Levina \& Zhu}]{Rothman01092010}
\textsc{Rothman, A.~J.}, \textsc{Levina, E.} \& \textsc{Zhu, J.} (2010).
\newblock A new approach to cholesky-based covariance regularization in high
  dimensions.
\newblock \textit{Biometrika} \textbf{97}, 539--550.

\bibitem[{Salakhutdinov \& Mnih(2008)}]{ruslan}
\textsc{Salakhutdinov, R.} \& \textsc{Mnih, A.} (2008).
\newblock Probabilistic matrix factorization.
\newblock In \textit{The Twenty-Second Annual Conference on Neural Information
  Processing Systems}. MIT Press.

\bibitem[{Scheinberg et~al.(2010)Scheinberg, Ma \&
  Goldfarb}]{Scheinberg_Ma_Goldfarb_2010}
\textsc{Scheinberg, K.}, \textsc{Ma, S.} \& \textsc{Goldfarb, D.} (2010).
\newblock Sparse inverse covariance selection via alternating linearization
  methods.
\newblock \textit{NIPS} , 1--9.

\bibitem[{Tseng(2001)}]{Tseng01}
\textsc{Tseng, P.} (2001).
\newblock Convergence of a block coordinate descent method for
  nondifferentiable minimization.
\newblock \textit{Journal of Optimization Theory and Applications}
  \textbf{109}, 475--494.

\bibitem[{WEN et~al.(2009)WEN, GOLDFARB \&
  Scheinberg}]{WEN_GOLDFARB_Scheinberg_2009}
\textsc{WEN, Z.}, \textsc{GOLDFARB, D.} \& \textsc{Scheinberg, K.} (2009).
\newblock Row by row methods for semidefinite programming.
\newblock \textit{Industrial Engineering} , 1--21.

\bibitem[{Yuan \& Lin(2007)}]{yuan_lin_07}
\textsc{Yuan, M.} \& \textsc{Lin, Y.} (2007).
\newblock Model selection and estimation in the gaussian graphical model.
\newblock \textit{Biometrika} \textbf{94}, 19--35.

\bibitem[{Yuan(2009)}]{Yuan_2009}
\textsc{Yuan, X.} (2009).
\newblock Alternating direction methods for sparse covariance selection.
\newblock
  \textit{at:http://www.optimization-online.org/DB-FILE/2009/09/2390.pdf} ,
  1--12.

\bibitem[{Zou \& Hastie(2005)}]{ZH2005}
\textsc{Zou, H.} \& \textsc{Hastie, T.} (2005).
\newblock Regularization and variable selection via the elastic net.
\newblock \textit{Journal of the Royal Statistical Society Series B.}
  \textbf{67}, 301--320.

\end{thebibliography}

\end{document}